\documentclass{jmlrmod}

\usepackage[small,compact]{titlesec}

\usepackage{dsfont,mathtools,enumitem,comment}
\usepackage{algorithm}
\usepackage{algorithmic}
\usepackage[T1]{fontenc}

\usepackage{float,lipsum}
\floatstyle{boxed}
\restylefloat{figure}

\usepackage{times}

\newcommand{\RR}{{\mathcal R}}

\newtheorem{infthm}{Informal Theorem}

\newlength{\figurewidth}
\newcommand{\II}{\mathds{1}}

\def\EE{{\mathbb{E}}}
\def\PP{{\mathbb{P}}}

\def\MI{{\mathcal{I}}}
\def\setR{\mathbb{R}}

\def\bZ{\mathbf{Z}}
\def\bz{\mathbf{z}}

\def\mut{\mathcal{I}}

\title[Price of Privacy]{The Price of Privacy in Untrusted Recommendation Engines}

 \author{
 \Name{Siddhartha Banerjee} \Email{sidb@stanford.edu}\\
 \addr Stanford University
 \AND
 \Name{Nidhi Hegde} \Email{nidhi.hegde@technicolor.com}\\
 \addr Technicolor, Paris Research Lab
 \AND
 \Name{Laurent Massouli{\'e}} \Email{laurent.massoulie@inria.fr}\\
 \addr Microsoft Research - INRIA Joint Center}

\begin{document}
\maketitle

\begin{abstract}
Recent increase in online privacy concerns prompts the following question: can a recommender system be accurate if users do not entrust it with their private data? To answer this, we study the problem of learning item-clusters under local differential privacy, a powerful, formal notion of data privacy. We develop bounds on the sample-complexity of learning item-clusters from privatized user inputs. Significantly, our results identify a sample-complexity separation between learning in an information-rich and an information-scarce regime, thereby highlighting the interaction between privacy and the amount of information (ratings) available to each user.
 
In the information-rich regime, where each user rates at least a constant fraction of items, a spectral clustering approach is shown to achieve a sample-complexity lower bound derived from a simple information-theoretic argument based on Fano's inequality. However, the information-scarce regime, where each user rates only a vanishing fraction of items, is found to require a fundamentally different approach both for lower bounds and algorithms. To this end, we develop new techniques for bounding mutual information under a notion of \emph{channel-mismatch}, and  also propose a new algorithm, \emph{MaxSense}, and show that it achieves optimal sample-complexity in this setting.
 
The techniques we develop for bounding mutual information may be of broader interest. To illustrate this, we show their applicability to $(i)$ learning based on 1-bit sketches, and $(ii)$ adaptive learning, where queries can be adapted based on answers to past queries.
\end{abstract}


\section{Introduction}
Recommender systems are fast becoming one of the cornerstones of the Internet; in a world with ever increasing choices, they are one of the most effective ways of matching users with items. Today, many websites use some form of such systems. Research in these algorithms received a fillip from the Netflix prize competition in 2009. Ironically, however, the contest also exposed the Achilles heel of such systems, when \cite{PrivAttack1} demonstrated that the Netflix data could be de-anonymized. Subsequent works (for example, \cite{PrivAttack2}) have reinforced belief in the frailty of these algorithms in the face of privacy attacks.

To design recommender systems in such scenarios, we first need to define what it means for a data-release mechanism to be private. The popular perception has coalesced around the notion that a person can either participate in a recommender system and waive all claims to privacy, or avoid such systems entirely. The response of the research community to these concerns has been the development of a third paradigm between complete exposure and complete silence. This approach has been captured in the formal notion of \emph{differential privacy} (refer \cite{Dwork06}); essentially it suggests that although perfect privacy is impossible, one can control the leakage of information by \emph{deliberately corrupting sensitive data before release}. The original definition in \cite{Dwork06} provides a statistical test that must be satisfied by a data-release mechanism to be private. Accepting this paradigm shifts the focus to designing algorithms that obey this constraint while maximizing relevant notions of utility. This trade-off between utility and privacy has been explored for several problems in database management \cite{Sulq,Dwork06,DworkMNS,Contobs,Panpriv} and learning \cite{BLR,PrivERM,SQbarrier,Kasietal,McsMiro,Smith}. 

In the context of recommender systems, there are two models for ensuring privacy: centralized and local. In the centralized model, the recommender system is trusted to collect data from users; it then responds to queries by publishing results that have been corrupted via some differentially private mechanism. However, users increasingly desire control over their private data, given their mistrust in centralized databases (which is supported by examples such as the Netflix privacy breach). In cases where the database cannot be trusted to keep data confidential, users can store their data locally, and differential privacy is ensured through suitable randomization at the `user-end' before releasing data to the recommender system. This is precisely the context of the present paper: the design of differentially private algorithms for untrusted recommender systems.

The latter model is variously known in privacy literature as \emph{local differential privacy} (see ~\cite{Kasietal}; we henceforth refer to it as \emph{local-DP} ), and in statistics as the `randomized response technique' (see~\cite{Warner}). However, there are two unique challenges to local-DP posed by recommender systems which have not been satisfactorily dealt with before:
\begin{enumerate}[nolistsep,noitemsep]
\item The underlying space (here, the set of ratings over \emph{all} items) has very high dimensionality.
\item The users have \emph{limited information}: they rate only a (vanishingly small) fraction of items.
\end{enumerate}
In this work we address both these issues. We consider the problem of learning an unknown (low-dimensional) clustering for a large set of items from privatized user-feedback. Surprisingly, we demonstrate a sharp change in the sample-complexity of local-DP learning algorithms when shifting from an information-rich to an information-scarce regime -- no similar phenomenon is known for non-private learning. With the aid of new information-theoretic arguments, we provide lower bounds on the sample-complexity in various regimes. On the other hand, we also develop novel algorithms, particularly in the information-scarce setting, which match the lower bounds up to logarithmic factors. Thus although we pay a `price of privacy' when ensuring local-DP in untrusted recommender systems with information-scarcity, we can design optimal algorithms for such regimes.

\subsection{Our Results}
We focus on learning a generative model for the data, under \emph{user-end, or local differential privacy} constraints. 
Local differential privacy ensures that user data is privatized before being made available to the recommender system -- the aim of the system is thus to learn the model from privatized responses to (appropriately designed) queries. The metric of interest is the \emph{sample-complexity} -- the minimum number of users required for efficient learning. 

Formally, given a set of items, we want to learn a partition or \emph{clustering} of the item-set, such that items within a cluster are statistically similar (in terms of user-ratings). The class of models (or \emph{hypothesis class}) we wish to learn is thus the set of mappings from items $[N]$\footnote{Throughout the paper, we use $[N]$ to denote the set $\{1,2,\ldots,N\}$.} to clusters $[L]$ (where typically $L<<N$). The system can collect information from $U$ users, where each user has rated only $w$ out of the $N$ items, and interacts with the system via a mechanism satisfying $\epsilon$-local-DP. To be deemed successful, we require that an algorithm \emph{identify the correct cluster label for all items}\footnote{This is for ease of exposition -- our results extend to allowing a fraction of item-misclassifications, c.f. Appendix \ref{app:misc}.}. 

To put the above model in perspective, consider the problem of movie-recommendation -- here items are movies, and the recommender system wants to learn a clustering of these movies, wherein two movies in a cluster are `similar'. We assume that each user has watched $w$ movies, but is unwilling to share these ratings with the recommender system without appropriate privatization of their data. Once the recommender system has learnt a good clustering, it can make this knowledge public, allowing users to obtain their own recommendations, based on their viewing history. This is similar in spirit to the `You Might Also Like' feature on IMDB or Amazon.

Our starting point for sample-complexity bounds is the following basic lower bound (c.f. Section \ref{sec:pre} for details):
\begin{infthm} 
\textbf{(Theorem~\ref{Thm:basiclb})} For any (finite) hypothesis class $\mathcal{H}$ to be `successfully' learned under $\epsilon$-local-DP, the number of users must satisfy:
$U_{LB}=\Omega\left(\frac{\log|\mathcal{H}|}{\epsilon}\right).$
\end{infthm}
\noindent The above theorem is based on a standard use of Fano's inequality in statistical learning. Similar connections between differential privacy and mutual information have been established before (c.f. Section \ref{ssec:relwork}) -- we include it here as it helps put our main results in perspective.

Returning to the recommender system problem, note that for the problem of learning item-clusters, $\log|\mathcal{H}|=\Theta (N)$. We next consider an \emph{information-rich setting}, wherein $w=\Omega(N)$, i.e., each user knows ratings for a constant fraction of the items. We show the above bound is matched (up to logarithmic factors) by a local-DP algorithm based on a novel `pairwise-preference' sketch and spectral clustering techniques:
\begin{infthm}
\textbf{(Theorem~\ref{Thm:pairpref})} In the information-rich regime under $\epsilon$-local-DP, clustering via the Pairwise-Preference Algorithm succeeds if the number of users satisfies: $U_{PP}^{IR}=\Omega\left(\frac{N\log N}{\epsilon}\right).$
\end{infthm}

The above theorems thus provide a complete picture of the information-rich setting. In practical scenarios, however, $w$ is quite small; for example, in a movie ratings system, users usually have seen and rated only a \emph{vanishing} fraction of movies. Our main results in the paper concern non-adaptive, local-DP learning in the \emph{information-scarce regime} -- wherein $w=o(N)$. Herein, we observe an interesting phase-change in the sample-complexity of private learning:
\begin{infthm}
In the information-scarce regime under $\epsilon$-local-DP, the number of users required for non-adaptive cluster learning must satisfy: $U_{LB}^{IS}=\Omega\left(\frac{N^2}{w^2}\right)$ \textbf{(Theorem~\ref{Thm:weaklb})}. 

\noindent Furthermore, for small $w$, in particular, $w=o(N^{\frac{1}{3}})$, we have: $U_{LB}^{IS}=\Omega\left(\frac{N^2}{w}\right)$ \textbf{(Theorem~\ref{Thm:stronglb})}.
\end{infthm}

To see why this result is surprising, consider the following toy problem: each item $i\in[N]$ belongs to one of two clusters. Users arrive, sample a \emph{single item} uniformly at random and learn its corresponding cluster, answer a query from the recommender system, and leave. 

For non-private learning, if there is no constraint on the amount of information exchanged between the user and the algorithm, then the number of users needed for learning the clusters is $\Theta\left(N\log N\right)$ (via a simple coupon-collector argument). Note that the amount of data each user has is $\Theta(\log N)$ (item index$+$cluster). Now if we put a constraint that the \emph{average amount of information} exchanged between a user and the algorithm is $1$ bit, then intuition suggests that the recommender system now needs $O\left(N\log^2 N\right)$ users. This is achieved by the following simple strategy: each user reveals her complete information with probability $\frac{1}{\log N}$, else reveals no information -- clearly the amount of information exchanged per user is $1$ bit on average, and a modified coupon collector argument shows that this scheme requires $O(N\log^2N)$ users to learn the item clusters. 

However, the situation changes if we impose a condition that the amount of information exchanged is \emph{exactly $1$ bit} per user (for example, the algorithm asks a yes/no question to the user); as a side-product of the techniques we develop for Theorem~\ref{Thm:stronglb}, we show that the number of users required in this case is $O(N^2)$ (c.f. Theorem~\ref{Thm:one-bit}). This fundamental change in sample-complexity scaling is due to the \emph{combination of users having limited information and a `per-user information' constraint} (as opposed to the average information constraint). One major takeaway of our work is that local differential privacy in the information-scarce regime has a similar effect.

Finally for the information-scarce regime, we develop a new algorithm, MaxSense, which (under appropriate separation conditions) matches the above bound up to logarithmic factors:
\begin{infthm}
\textbf{(Theorem~\ref{Thm:maxsense})} In the information-scarce regime under $\epsilon$-local-differential-privacy, for given $w=o(N)$, clustering via the MaxSense Algorithm (Section~\ref{sec:maxsense}) is successful if the number of users satisfies: $U_{MS}=\Omega\left(\frac{N^2\log N}{w\epsilon}\right).$
\end{infthm}

\noindent\textbf{Techniques:} Our main technical contribution lies in the tools we use for the lower bounds in the information-scarce setting. By viewing the privacy mechanism as a noisy channel with appropriate constraints, we are able to use information theoretic methods to obtain bounds on private learning. Although connections between privacy and mutual information have been considered before (refer \cite{2party,Renyi}), existing techniques do not capture the change in sample-complexity in high-dimensional regimes. We formalize a new notion of `channel mis-alignment' between the `sampling channel' (the partial ratings known to the users) and the privatization channel. In Section~\ref{sec:LBscarce} we provide a structural lemma (Lemma~\ref{lemma:Ibnd}) that quantifies this mismatch under general conditions, and demonstrate its use by obtaining tight lower bounds under $1$-bit (non-private) sketches. In Section~\ref{ssec:LB} we use it to obtain tight lower bounds under local-DP. In Section~\ref{ssec:adaptive} we discuss its application to adaptive local-DP algorithms, establishing a lower bound of order $\Omega(N\log N)$ -- note that this again is a refinement on the bound in Theorem~\ref{Thm:basiclb}. Though we focus on the item clustering problem, our lower bounds \emph{apply to learning any finite hypothesis class under privacy constraints}.

The information theoretic results also suggest that $1$-bit privatized sketches are sufficient for learning in such scenarios. Based on this intuition, we show how existing spectral-clustering techniques can be extended to private learning in some regimes. More significantly, in the information-scarce regime, where spectral learning fails, we develop a novel algorithm based on blind probing of a large set of items. This algorithm, in addition to being private and having optimal sample-complexity in many regimes, suggests several interesting open questions, which we discuss in Section~\ref{sec:extn}.

\subsection{Related Work}
\label{ssec:relwork}
\textbf{Privacy preserving recommender systems:} The design of recommender systems with differential privacy was studied by \cite{McsMiro} under the centralized model. Like us, they separate the recommender system into two components, a learning phase (based on a database appropriately perturbed to ensure privacy) and a recommendation phase (performed by the users `at home', without interacting with the system). They numerically compare the performance of the algorithm against non-private algorithms. In contrast, we consider a stronger notion of privacy (local-DP), and for our generative model, are able to provide tight analytical guarantees and further, quantify the impact of limited information on privacy.

\noindent\textbf{Private PAC Learning and Query Release:} Several works have considered private algorithms for PAC-learning. \cite{BLR,SQbarrier} consider the private query release problem (i.e., releasing approximate values for all queries in a given class) in the centralized model. \cite{Kasietal} show equivalences between: a) centralized private learning and agnostic PAC learning, b) local-DP and the statistical query (SQ) model of learning; this line of work is further extended by \cite{beimel}. Although some of our results (in particular, Theorem \ref{Thm:basiclb}) are similar in spirit to lower bounds for PAC (see \cite{Kasietal,beimel} there are significant differences both in scope and technique. Furthermore:  
\begin{enumerate}[nolistsep,noitemsep]
\item We emphasize the importance of limited information, and characterize its impact on learning with local-DP. Hitherto unconsidered,information scarcity is prevalent in practical scenarios, and as our results shows, it has strong implications on learning performance under local-DP.
\item Via lower bounds, we provide a tight characterization of sample-complexity, unlike~\cite{Kasietal,BLR,SQbarrier}, which are concerned with showing polynomial bounds. This is important for high dimensional data.
\end{enumerate}

\noindent\textbf{Privacy in Statistical Learning:} 
A large body of recent work has looked at the impact of differential privacy on statistical learning techniques. A majority of this work focusses on centralized differential privacy. For example, ~\cite{PrivERM} consider privacy in the context of empirical risk minimization; they analyze the release of classifiers, obtained via algorithms such as SVMs, with (centralized) privacy constraints on the training data.\cite{Robust} study algorithms for privacy-preserving regression under the centralized model; these however require running time which is exponential in the data dimension. \cite{Smith} obtains private, asymptotically-optimal algorithms for statistical estimation, again though, in the centralized model.

More recently, \cite{Duchi2013} consider the problem of finding minimax rates for statistical estimators under local-DP. Their techniques are based on refined analysis of information theoretic quantities, including generalizations of the Fano's Inequality bounds we use in Section \ref{ssec:LBbasic}. However, the estimation problems they consider have a simpler structure -- in particular, they involve learning from samples generated directly from an underlying model (albeit privatized). What makes our setting challenging is the combination of a generative model (the bipartite stochastic blockmodel) with incomplete information (due to user-item sampling) -- it seems unlikely that the techniques of \cite{Duchi2013} can extend easily to our setting. Moreover, lower bound techniques do not naturally yield good algorithms

\noindent\textbf{Other Notions of Privacy:} The local-DP model which we consider has been studied before in privacy literature~(\cite{Kasietal,DworkMNS}) and statistics~(\cite{Warner}). It is a stronger notion than central differential privacy, and also stronger than two other related notions: pan-privacy~(\cite{Panpriv}) where the database has to also deal with occasional release of its state, and privacy under continual observations~(\cite{Contobs}), where the database must deal with additions and deletions, while maintaining privacy.  

\noindent\textbf{Recommendation algorithms based on incoherence:} Apart from privacy-preserving algorithms, there is a large body of work on designing recommender systems under various constraints (usually low-rank) on the ratings matrix~(for example,~\cite{Wain09, KMOnoise}). These methods, though robust, fail in the presence of privacy constraints, as the noise added as a result of privatization is much more than their noise-tolerance. This is intuitive, as successful matrix completion would constitute a breach of privacy; our work builds the case for using simpler lower dimensional representations of the data, and simpler algorithms based on extracting limited information (in our case, $1$-bit sketches) from each user.

\section{Preliminaries} 
\label{sec:pre}

We now present our system model, formally define different notions of differential privacy, and introduce some tools from information theory that form the basis of our proofs. 

\subsection{The Bipartite Stochastic BlockModel}
\label{ssec:model}

Recommender system typically assume the existence of an underlying low-dimensional generative model for the data -- the aim then is to learn parameters of this model, and then, use the learned model to infer unknown user-item rankings. In this paper we consider a model wherein items and users belong to underlying clusters, and a user's ratings for an item depend only on the clusters they belong to. This is essentially a bipartite version of the \emph{Stochastic Blockmodel} \cite{StochBlock}, widely used in model selection literature. The aim of the recommendation algorithm is to learn these clusters, and then reveal them to the users, who can then compute their own recommendations privately. Our model, though simpler than the state of the art in recommender systems, is still rich enough to account for many of the features seen empirically in recommender systems. In addition it yields reasonable accuracy in non-private settings on meaningful datasets (c.f.~\cite{TomoMass}). 

Formally, let $[U]$ be the set of $U$ users and $[N]$ the set of $N$ items. The set of users is divided into $K$ clusters $[K]$, where cluster $i$ contains $\alpha_i U$ users. Similarly, the set of items is divided into $L$ clusters $[L]$, where cluster $\ell$ contains $\beta_{\ell} N$ items. We use $A$ to denote the (incomplete) matrix of user/item ratings, where each row corresponds to a user, and each column an item. For simplicity, we assume $A_{ij}\in\{0,1\}$; for example, this could correspond to `like/dislike' ratings. Finally we have the following statistical assumption for the ratings -- for user $u\in[U]$ with user class $k$, and item $i\in [N]$ with item class $\ell$, the rating $A_{ui}$ is given by a Bernoulli random variable  $A_{ui}\sim\mbox{Bernoulli}(b_{k\ell})$. Ratings for different user-item pairs are assumed independent.

In order to model limited information, i.e., the fact that users rate only a fraction of all items, we define a parameter $w$ to be the number of items a user has rated. More generally, we only need to know $w$ in an orderwise sense -- for example, $w=\Theta(f(N))$ for some function $f$. We assume that the rated items are picked uniformly at random. We define \emph{$w=\Omega(N)$ to be the information-rich regime, and $w=o(N)$ to be the information-scarce regime}.

Given this model, the aim of the recommender system is to learn the item-clusters from user-item ratings. Note that the difficulty in doing so is twofold:
\begin{itemize}[nolistsep,noitemsep]
\item The user-item ratings matrix $A$ is incomplete -- in particular, each user has ratings for only $w$ out of $N$ items.
\item Users share their information only via a privacy-preserving mechanism  (as we discuss in the next section). 
\end{itemize}
Our work exposes how these two factors interact to affect the \emph{sample-complexity}, i.e., the minimum number of users required to learn the item-clusters. We note also that another difficulty in learning is that the user-item ratings are noisy -- however, as long as this noise does not depend on the number of items, this does not affect the sample-complexity scaling.

\subsection{Differential Privacy}
\label{ssec:DP}

Differential privacy is a framework that defines conditions under which an algorithm can be said to be privacy preserving with respect to the input. Formally (following \cite{Dwork06}):
\begin{definition}
\textbf{($\epsilon$-Differential Privacy)} A randomized function $\Psi:\mathcal{X}\rightarrow\mathcal{Y}$ that maps data $X\in\mathcal{X}$ to $Y\in\mathcal{Y}$ is said to be $\epsilon$-differentially private if, for all values $y\in\mathcal{Y}$ in the range space of $\Psi$, and for all `neighboring' data $x,x'$, we have:
\begin{equation}
\label{eq:DP}
\frac{\PP[Y=y|X=x]}{\PP[Y=y|X=x']}\leq e^{\epsilon}
\end{equation}
\end{definition}  
We assume that $Y$ conditioned on $X$ is independent of any external side information $Z$ (in other words, the output of mechanism $\Psi$ depends only on $X$ and its internal randomness). The definition of `neighboring' is chosen according to the situation, and determines the data that remain private. In the original definition \cite{Dwork06}, two databases are said to be neighbors if the larger database is constructed by adding a single tuple to the smaller database. In the context of ratings matrices, two matrices can be neighbors if they differ in: $i)$ a single row (per-user privacy), or $ii)$ a single rating (per-rating privacy). 

Two crucial properties of differential privacy are \emph{composition} and \emph{post-processing}. We state these here without proof; c.f. \cite{Dwork06} for details. Composition captures the reduction in privacy due to sequentially applying multiple differentially-private release mechanisms:
\begin{proposition} 
\label{prop:composition}
\textbf{(Composition)} If $k$ outputs, $\{Y_1,Y_2,\ldots,Y_k\}$ are obtained from data $X\in\mathcal{X}$ by $k$ different randomized functions, $\{\Psi_1,\Psi_2,\ldots,\Psi_k\}$, where $\Psi_i$ is $\epsilon_i$-differentially private, then the resultant function is $\sum_{i=1}^k\epsilon_i$ differentially private.
\end{proposition}
\noindent Post-processing states that processing the output of a differentially private release mechanism can only make it more differentially private (i.e., with a smaller $\epsilon$) vis-a-vis the input:
\begin{proposition}
\label{prop:postproc}
\textbf{(Post-processing)} If a function $\Psi_1:\mathcal{X}\rightarrow\mathcal{Y}$ is $\epsilon$-differentially private, then any composition function $\Psi_2\circ\Psi_1:\mathcal{X}\rightarrow\mathcal{Z}$ is $\epsilon'$-differentially private for some $\epsilon'\leq\epsilon$.
\end{proposition} 
 
In settings where the database curator is untrusted, an appropriate notion of privacy is \emph{local differential privacy} (or local-DP). For each user $u$, let $X_u$ be its private data -- in the recommendation context, the rated-item labels and corresponding ratings -- and let $Y_u$ be the data that the user makes publicly available to the untrusted curator. Local-DP requires that $Y_u$ is $\epsilon$ differentially private w.r.t. $X_u$. This paradigm is similar to the Randomized Response technique in statistics~\cite{Warner}. It is the natural notion of privacy in the case of untrusted databases, as the data is privatized \emph{at the user-end before storage in the database}; to emphasize this, we alternately refer to it as \emph{User-end Differential Privacy}. 

We conclude this section with a mechanism for releasing a single bit under $\epsilon$-differential privacy. Differential privacy for this mechanism is easy to verify using equation~\ref{eq:DP}.
\begin{proposition}
\label{prop:1bitpriv}
\textbf{($\epsilon$-DP bit release):} Given bit $S^0\in\{0,1\}$, set output $S$ to be equal to $S^0$ with probability $\frac{e^{\epsilon}}{1+e^{\epsilon}}$, else equal to $\overline{S}^0=1-S^0$. Then $S$ is $\epsilon$-differentially private w.r.t. $S^0$.
\end{proposition}

\subsection{Preliminaries from Information Theory} 
\label{subsec:ITbasics}

For a random variable $X$ taking values in some discrete space $\mathcal{X}$, its entropy is defined as $H(X):=\sum_{x\in\mathcal{X}}-\PP[X=x]\log \PP[X=x]$ \footnote{For notational convenience, we use $\log(\cdot)$ as the logarithm to the base $2$ throughout; hence, the entropy is in `bits'}. For two random variables $X,Y$, the mutual information between them is given by:
\begin{equation*}
\MI(X;Y):=\sum_{(x,y)}\PP[X=x,Y=x]\log\left(\frac{\PP[X=x,Y=y]}{\PP[X=x]\PP[Y=y]}\right)\hspace{0.5cm}.
\end{equation*}

Our main tools for constructing lower bounds are variants of Fano's Inequality, which are commonly used in non-parametric statistics literature (c.f. \cite{SantWain,Wain09}). Consider a finite hypothesis class $\mathcal{H}, |\mathcal{H}|=M$, indexed by $[M]$. Suppose that we choose a hypothesis $H$ uniformly at random from $\{1,2,\ldots,M\}$, sample a data set $\mathbf{X}_1^U$ of $U$ samples drawn in an i.i.d. manner according to a distribution $P_{\mathcal{H}}(H)$ (in our case, $u\in[U]$ corresponds to a user, and $X_u$ the ratings drawn according to the statistical model in Section~\ref{ssec:model}), and then provide a private version of this data $\widehat{\mathbf{X}}_1^U$ to the learning algorithm. We can represent this as the Markov chain:
\begin{equation*}
H\in\mathcal{H}\xrightarrow{\mbox{Sampling}} \mathbf{X}_1^U\xrightarrow{\mbox{Privatization}}\widehat{\mathbf{X}}_1^U\xrightarrow[\mbox{Selection}]{\mbox{Model}}\widehat{H}
\end{equation*}
Further, we define a given learning algorithm to be \emph{unreliable} for the hypothesis class $\mathcal{H}$ if for a hypothesis drawn uniformly at random, we have $\max_{h\in[M]}\PP\left[\widehat{H}\neq H| H=h\right]>\frac{1}{2}$.

Fano's inequality provides a lower bound on the probability of error under any learning algorithm in terms of the mutual information between the underlying hypotheses and the samples. A basic version of the inequality is as follows:
\begin{lemma} \textbf{(Fano's Inequality)}
\label{lemma:fano}
Given a hypothesis $H$ drawn uniformly from $\mathcal{H}$, and $U$ samples $\mathbb{X}_1^U$ drawn according to $H$, for any learning algorithm, the average probability of error $P_e:=\PP[\widehat{H}\neq H]$ satisfies:
\begin{equation}
\label{eq:fano}
P_e\geq 1-\frac{\mathcal{I}(H;\mathbf{X}_1^U)+1}{\log\left(M\right)} \, .
\end{equation}
\end{lemma}

As a direct consequence of this result, if the samples are such that $\mathcal{I}(H;\mathbf{X}_1^U)=o(\log M)$, then any algorithm fails to correctly identify \emph{almost all} of the possible underlying models. Though this is a weak bound, equation \ref{eq:fano} turns out to be sufficient to study sample-complexity scaling in the cases we consider. In 
Appendix \ref{app:misc}, we consider stronger versions of the above lemma, as well as more general criterion for approximate model selection (e.g., allowing for distortion).

\section{Item-Clustering under Local-DP: The Information-Rich Regime}
\label{sec:rich}

In this section, we derive a basic lower bound on the number of users needed for accurate learning under local differential privacy. This relies on a simple bound on the mutual information between any database and its privatized output, and hence is applicable in general settings. Returning to item-clustering, we give an algorithm that matches the optimal scaling (up to logarithmic factor) under one of the following two conditions: $i)$ $w=\Omega(N)$, i.e., each user has rated a constant fraction of items (the information-rich regime), or $ii)$ only the ratings are private, not the identity of the rated items.

\subsection{Differential Privacy and Mutual Information}
\label{ssec:LBbasic}

We first present a lemma that characterizes the mutual information leakage across any differentially private channel: 

\begin{lemma}
\label{lemma:basicIbnd}
Given (private) r.v. $X\in\mathcal{X}$, a privatized output $Y\in\mathcal{Y}$ obtained by any locally $\epsilon-$DP mechanism $\Phi:\mathcal{X}\rightarrow\mathcal{Y}$, and any side information $Z$, we have:
$I(X;Y|Z)\leq\epsilon\log e.$
\end{lemma}

Lemma~\ref{lemma:basicIbnd} follows directly from the definitions of mutual information and differential privacy (note that for any such mechanism, the output $Y$ given the input $X$ is conditionally independent of any side-information). We note that similar results have appeared before in literature; for example, equivalent statements appear in \cite{2party,Renyi}. We present the proof here for the sake of completeness:
\begin{proof}[Proof of Lemma \ref{lemma:basicIbnd}]
\begin{align*}
I(X;Y|Z) &=\sum_{(x,y)\in\mathcal{X}\times\mathcal{Y}}p(x,y|Z)\log{\left[\frac{p(x,y|Z)}{p(x|Z)p(y|Z)}\right]}\\
&=\sum_{(x,y)\in\mathcal{X}\times\mathcal{Y}}p(x,y|Z)\log\left[\frac{p(y|x,Z)}{\sum_{x'\in\mathcal{X}}p(x'|Z)p(y|x',Z)}\right]\\
&=\sum_{(x,y)\in\mathcal{X}\times\mathcal{Y}} -p(x,y|Z)
\log\left[\sum_{x'\in\mathcal{X}}p(x'|Z)\frac{p(y|x',Z)}{p(y|x,Z)}\right]\\
&\stackrel{(a)}{\leq}\sum_{(x,y)\in\mathcal{X}\times\mathcal{Y}}-p(x,y|Z)
\log\left[\sum_{x'\in\mathcal{X}}p(x'|Z)e^{-\epsilon}\right]\leq\epsilon\log e.
\end{align*}
Here inequality $(a)$ is a direct application of the definition of differential privacy (Equation \ref{eq:DP}), and in particular, the fact that it holds for any side information.
\end{proof}

Returning to the private learning of item classes, we obtain a lower bound on the sample-complexity by considering the following special case of the item-clustering problem: consider $\mathcal{H}=\{0,1\}^N$, and let $C_N\in\mathcal{H}$ be a mapping of the item set $[N]$ to \emph{two classes} represented as $\{0,1\}$ -- hence the size of the hypothesis class is $2^N$. Each user $u$ has some private data $X_u$, which is generated via the bipartite Stochastic Blockmodel (c.f., Section \ref{ssec:model}). Recall we define a learning algorithm to be \emph{unreliable} for $\mathcal{H}$ if $\max_{h\in\mathcal{H}}\PP\left[\widehat{C_N}\neq C_N| C_N=h\right]>\frac{1}{2}$. Using Lemma~\ref{lemma:basicIbnd} and Fano's inequality (Lemma~\ref{lemma:fano}), we get the following lower bound on the sample-complexity:
\begin{theorem} 
\label{Thm:basiclb} 
Suppose the underlying clustering $C_N$ is drawn uniformly at random from $\{0,1\}^N$. Then any learning algorithm obeying $\epsilon$-local-DP is unreliable if the number of queries satisfies:
$U<\left(\frac{N}{\epsilon\log e}\right)$.
\end{theorem}

\begin{proof}
We now have the following information-flow model for \emph{each user} (under local-DP):
$$C_N\xrightarrow{\mbox{Sampling}} X_u\xrightarrow{\mbox{Privatization}}\widehat{X}_u$$
Here sampling refers to each user rating a subset of $w$ items. Now by using the Data-Processing Inequality 
(Theorem $2.8.1$ from \cite{CoverThomas}), followed by Lemma~\ref{lemma:basicIbnd}, we have that: 
$$\MI(C_N;\widehat{X}_1^U)\leq \sum_{u=1}^U \MI(X_u;\widehat{X_u}|\widehat{\mathbf{X}_1^{u-1}})< U \epsilon\log e,$$ 
Fano's inequality (Lemma~\ref{lemma:fano}) then implies that a learning algorithm is unreliable if the number of queries satisfies:  $U <\left(\frac{N}{\epsilon\log e}\right)$.
\end{proof}
We note here that the above theorem, though stated for the bipartite Stochastic Blockmodel, in fact gives sample-complexity bounds for more general model-selection problems. Further, in Appendix \ref{app:misc}, we extend the result to allow for \emph{distortion} -- wherein the algorithm is allowed to make a mistake on some fraction of item-labels.  

For the bipartite Stochastic Blockmodel, though the above bound is not the tightest, it turns out to be achievable (up to log factors) in the information-rich regime, as we show next. We note that a similar bound was given by \cite{beimel} for PAC-learning under centralized DP, using more explicit counting techniques. Both our results and the bounds in \cite{beimel} fail to exhibit the correct scaling in the information-scarce case ($w=o(N)$) setting. However, unlike proofs based on counting arguments, our method allows us to leverage more sophisticated information theoretic tools for other variants of the problem, like those we consider subsequently in Section \ref{sec:LBscarce}.

\subsection{Item-Clustering in the Information-Rich Regime}
\label{ssec:PP}

To conclude this section, we outline an algorithm for clustering in the information-rich regime. The algorithm proceeds as follows: $i)$ the recommendation algorithm provides each user $u$ with two items $(i_u,j_u)$ picked at random, whereupon the user computes a private sketch $S^0_u$ which is equal to $1$ if she rated the two items positively, and else $0$, $ii)$ users release a privatized version $S_u$ of their private sketch using the $\epsilon$-DP bit release mechanism, $iii)$ the algorithm constructs matrix $\widehat{A}$, where $\widehat{A}(i,j)$ entry is obtained by adding the sketches from all users queried with item-pair $(i,j)$, and finally $iv)$ performs spectral clustering of items based on matrix $\widehat{A}$. This algorithm, which we refer to as the Pairwise-Preference algorithm, is formally specified in Figure \ref{alg:pairpref}.

\begin{figure}[h]
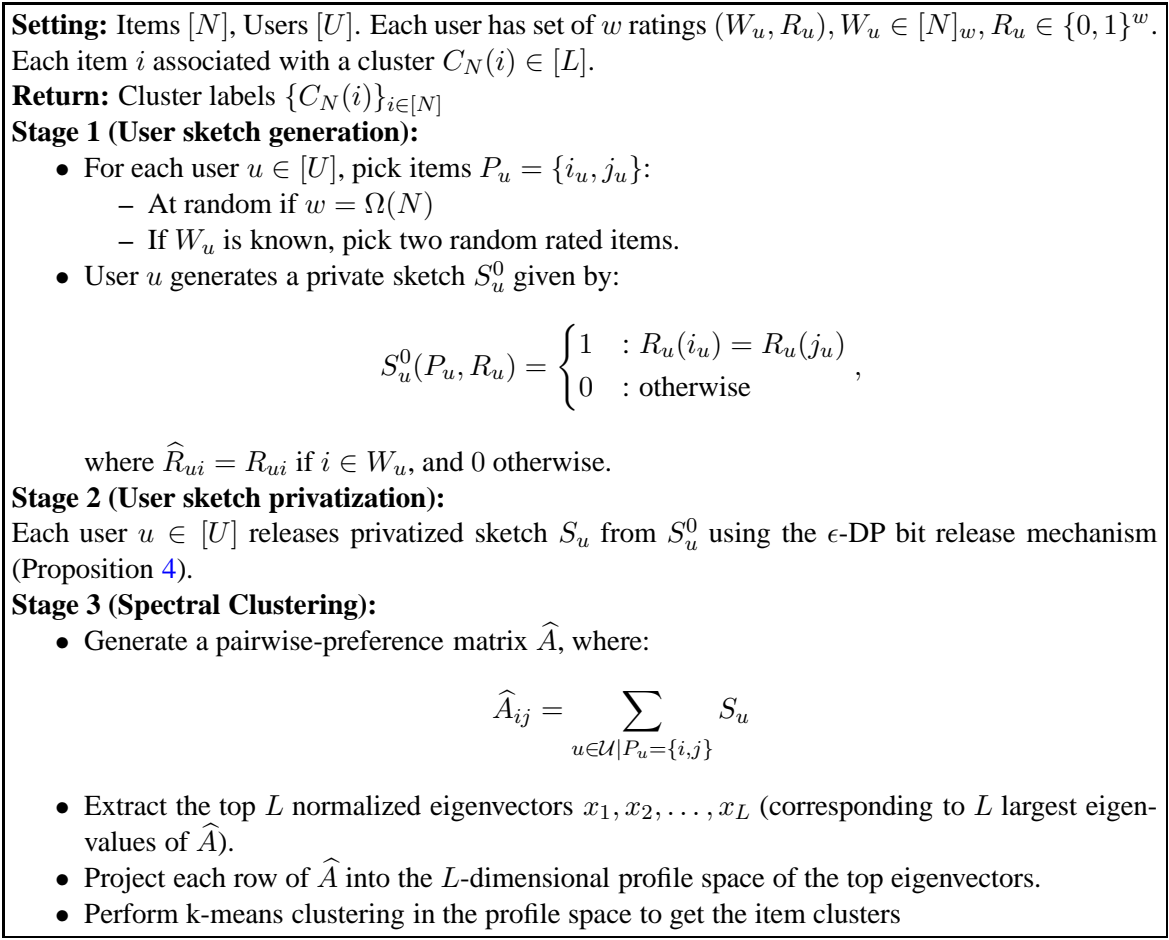

\noindent\textbf{Setting:} Items $[N]$, Users $[U]$. Each user has set of $w$ ratings $(W_u,R_u), W_u\in[N]_w, R_u\in\{0,1\}^w$. Each item $i$ associated with a cluster $C_N(i)\in[L]$.

\noindent\textbf{Return:} Cluster labels $\{C_N(i)\}_{i\in[N]}$

\noindent\textbf{Stage 1 (User sketch generation):} 
\begin{itemize}[nolistsep,noitemsep]
\item For each user $u\in[U]$, pick items $P_u=\{i_u,j_u\}$:
\begin{itemize}[nolistsep,noitemsep]
\item At random if $w=\Omega(N)$
\item If $W_u$ is known, pick two random rated items.
\end{itemize} 
\item User $u$ generates a private sketch $S_u^0$ given by:
\begin{equation*}
\label{eq:privpp}
S_u^0(P_u,R_u)=
\begin{dcases*} 
1 & : $R_u(i_u)=R_u(j_u)$\\
0 & : otherwise
\end{dcases*},
\end{equation*}
where $\widehat{R}_{ui}=R_{ui}$ if $i\in W_u$, and $0$ otherwise.
\end{itemize}

\noindent\textbf{Stage 2 (User sketch privatization):}\\
Each user $u\in[U]$ releases privatized sketch $S_u$ from $S_u^0$ using the $\epsilon$-DP bit release mechanism (Proposition~\ref{prop:1bitpriv}).

\noindent\textbf{Stage 3 (Spectral Clustering):}
\begin{itemize}[nolistsep,noitemsep]
\item Generate a pairwise-preference matrix $\widehat{A}$, where:
\begin{equation*}
\widehat{A}_{ij}=\sum_{u\in\mathcal{U}|P_u=\{i,j\}}S_u
\end{equation*}
\item Extract the top $L$ normalized eigenvectors $x_1,x_2,\ldots, x_L$ (corresponding to $L$ largest eigenvalues of $\widehat{A}$).
\item Project each row of $\widehat{A}$ into the $L$-dimensional profile space of the top eigenvectors.
\item Perform k-means clustering in the profile space to get the item clusters 
\end{itemize}
\caption{The Pairwise-Preference Algorithm}
\label{alg:pairpref}
\end{figure}

Recall in the bipartite Stochastic Blockmodel, we assume that the $U$ users belong tp $K$ clusters, each of size $\alpha_i U$. We now have the following theorem that characterizes the performance of the Pairwise-Preference algorithm.
\begin{theorem} 
\label{Thm:pairpref}
The Pairwise-Preference algorithm 
satisfies $\epsilon$-local-DP. Further, suppose the eigenvalues and eigenvectors of $\widehat{A}$ satisfy the following non-degeneracy conditions:
\begin{itemize}
\item The $L$ largest magnitude eigenvalues of $A$ have distinct absolute values.
\item The corresponding eigenvectors $y_1,y_2,\ldots,y_L$, normalized under the $\alpha$-norm, $||y||^2_{\alpha}=\sum_{k=1}^K\alpha_ky_k^2$, for some $\alpha$ satisfy:
\begin{equation*}
t_i\neq t_j\quad,1\leq i<j\leq L
\end{equation*}
where $t_i:= (y_1(i),\ldots ,y_L(i))$.
\end{itemize} 
Then, in the information-rich regime (i.e., when $w=\Omega(N)$), there exists $c>0$ such that the item clustering is successful with high probability if the number of users satisfies:
\begin{equation*}
U\geq c\left(N\log N\right) \, .
\end{equation*}
\end{theorem}
 
\begin{proof}[Proof Outline]
Local differential privacy under the Pairwise-Preference algorithm is guaranteed by the use of $\epsilon$-DP bit release, and the composition property. The performance analysis is based on a result on spectral clustering by \cite{TomoMass}. The main idea is to interpret $\widehat{A}$ as representing the edges of a random graph over the item set, with an edge between an item in class $i$ and another in class $j$ if $\widehat{A}_{ij}>0$. In particular, from the definition of the Pairwise Preference algorithm, we can compute that the probability of such an edge is $\Theta\left(\frac{b_{ij}\log N}{N}\right)$. This puts us in the setting analyzed by \cite{TomoMass} -- we can now use their spectral clustering bounds to get the result. For the complete proof, refer Appendix \ref{app:PP}.
\end{proof}

\section{Local-DP in the Information-Scarce Regime: Lower Bounds}
\label{sec:LBscarce}

As in the previous lower bound, we consider a simplified version of the problem, where there is a single class of users, and each item is ranked either $0$ or $1$ deterministically by each user (i.e., $b_{ui}=b_i\in\{0,1\}$ for all items). Let $C_N(\cdot):[N]\rightarrow\{0,1\}$ be the underlying clustering function; in general we can think of this as an $N$-bit vector $\bZ\in\{0,1\}^N$. We assume that the user-data for user $u$ is given by $X_u=(I_u,Z_u)$, where $I_u$ is a size $w$ subset of $[N]$ representing items rated by user $u$, and $Z_u$ are the ratings for the corresponding items; in this case, $Z_u=\{\bZ(i)\}_{i\in I_u}$. The set $I_u$ is assumed to be chosen uniformly at random from amongst all size-$w$ subsets of $[N]$. We also denote the privatized sketch from user $u$ as $S_u\in\mathcal{S}$. Here the space ${\mathcal S}$ to which sketches belong is assumed to be an arbitrary finite or countably infinite space. The sketch is assumed $\epsilon$-differentially private. Finally, as before, we assume that $\bZ$ is chosen uniformly over $\{0,1\}^N$. Thus we have the following information-flow model for the user $u$:
\begin{equation*}
\bZ\xrightarrow{\mbox{Sampling}} (I_u,Z_u)\xrightarrow{\mbox{Privatization}}S_u
\end{equation*} 

Now to get tighter lower bounds on the number of users needed for accurate item clustering, we need more accurate bounds on the mutual information between the underlying model on item-clustering and the data available to the algorithm. 
The main idea behind our lower bound techniques is to view the above chain as a combination of two channels -- the first wherein the user-data $(I_u,Z_u)$ is generated (sampled) by the underlying statistical model, and the second wherein the algorithm receives a sketch $S_u$ of the user's data. We then develop a new information inequality that allows us to \emph{bound the mutual information in terms of the mismatch between the channels}. This technique turns out to be useful in settings without privacy as well -- in Section \ref{ssec:1bit}, we show how it can be used to get sample-complexity bounds for learning with $1$-bit sketches.

\subsection{Mutual Information under Channel Mismatch}
\label{ssec:2lemmas}

We now establish a bound for the mutual information between a statistical model and a low-dimensional sketch, which is the main tool we use to get sample-complexity lower bounds. We define $[N]_w$ to be the collection of all size-$w$ subsets of $[N]$, and $\mathcal{D}:=[N]_w\times \{0,1\}^w$ to be the set from which user information (i.e., $(I,Z)$) is drawn, and define $D=|\mathcal{D}|=\binom{N}{w}2^w$. Finally $\EE_X[\cdot]$ indicates that the expectation is over the random variable $X$.  
\begin{lemma}
\label{lemma:Ibnd}
Given the Markov Chain $\bZ\rightarrow(I,Z)\rightarrow S$, let $(I_1,Z_1),(I_2,Z_2)\in\mathcal{D}$ be two pairs of `user-data' sets which are independent and identically distributed according to the conditional distribution of the pair $(I,Z)$ given $S=s$. Then, the mutual information $\mut(\bZ;S)$ satisfies:
\begin{equation*}
\mut(\bZ;S)\le \EE_S\left[\EE_{(I_1,Z_1)|S\perp \! \! \! \perp(I_2,Z_2)|S}\left[2^{|I_1\cap I_2|}\II_{\{Z_1\equiv Z_2\}}-1\right]\right],
\end{equation*}
where we use the notation $\II_{\{Z_1\equiv Z_2\}}$ to denote that the two user-data sets are consistent on the index set on which they overlap, i.e., $\II_{\{Z_1\equiv Z_2\}}:=\II_{\{Z_1(\ell)=Z_2(\ell)\forall \ell \in I_1\cap I_2\}}$
\end{lemma}

\begin{proof}	
For brevity, we use the shorthand notation $p(\bz)=\PP[\bZ=\bz], p(s)=\PP[S=s], p(\bz|s)=\PP[\bZ=\bz|S=s]$ and finally $p(\bz,s)=\PP[(\bZ,S)=(\bz,s)]$. Now we have:
\begin{align}
\mut(\bZ;S)
&=\sum_{\bz,s}p(\bz,s)\log\left(\frac{p(\bz,s)}{p(\bz)p(s)}\right)\nonumber\\
&=\sum_{s}\sum_{\bz}p(\bz|s)p(s) \log\left(\frac{p(\bz|s)}{p(\bz)}\right)\nonumber\\
\label{Ibnd1} &\leq\EE_S\left[\sum_{\bz}p(\bz|s)\log\left(\frac{p(\bz|s)}{p(\bz)}\right)\right]
\end{align}
Let $f(s,z):=\log\left(\frac{p(\bz,s)}{p(\bz)}\right)$. Similar to above, we use the shorthand notation $p(\cdot|a,b):=\PP[\cdot|A=a,B=b]$, where $(A,B)$ are random variables and $(a,b)$ their corresponding realizations. Now we have:
\begin{align}
f(s,z)&=\log\left(\frac{p(\bz|s)}{p(\bz)}\right)
=\log\left(\frac{\sum_{(i_2,z_2)}p(\bz,i_2,z_2|s)}{p(\bz)}\right)\nonumber\\
&\mbox{(Summing over $(I_2,Z_2)$)}\nonumber\\
&=\log\left(\frac{\sum_{(i_2,z_2)}p(\bz|i_2,z_2,s)p(i_2,z_2|s)}{p(\bz)}\right)\nonumber\\
&=\log\left(\frac{\sum_{(i_2,z_2)}p(\bz|i_2,z_2)p(i_2,z_2|s)}{p(\bz)}\right)\nonumber\\
&\mbox{(By the Markov property)}\nonumber\\
&\leq\frac{\sum_{(i_2,z_2)}p(\bz|i_2,z_2)p(i_2,z_2|s)}{p(\bz)}-1\nonumber\\
&\mbox{(Since $\log(x)\leq x-1$)}\nonumber\\
&=\sum_{(i_2,z_2)}p(i_2,z_2|s)2^{|i_2|}\II_{\{z\equiv z_2\}}-1,\label{Ibnd2}
\end{align}
where the last equality is obtained using the fact that the type of each item $\bZ(l)$ is independent and uniformly distributed over $\{0,1\}$. Next, using a similar set of steps, we have:
\begin{align}
\sum_{\bz}p(\bz|s)f(s,z)&=\sum_{\bz}\sum_{i_1}p(\bz,i_1|s)f(s,z)\nonumber\\
&=\sum_{i_1,z_1}\sum_{z_{-i_1}}p(i_1,z_1|s)\PP[Z_{-I_1}=z_{-i_1}|i_1,z_1]f(s,z)\nonumber\\
&\mbox{$\Big($Where $Z_{-i_1}:=\{Z(l)|l\in[N]\setminus i_1\}\in\{0,1\}^{N-|i_1|}\Big)$}\nonumber\\
&=\sum_{i_1,z_1}\sum_{z_{-i_1}}p(i_1,z_1|s)2^{-(N-|i_1|)}f(s,z),
\label{Ibnd3}
\end{align}
Finally, we combine equations (\ref{Ibnd1}),(\ref{Ibnd2}) and (\ref{Ibnd3}) together to get the result:
\begin{align*}
\mut(\bZ;S)&\leq\EE_S\left[\sum_{\bz}\PP[\bZ=\bz|S=s]f(s,z)\right]\\
&\leq\EE_S\left[\sum_{(i_1,z_1)}\sum_{z_{-i_1}}p(i_1,z_1|s)2^{-(N-|i_1|)}\left(\sum_{(i_2,z_2)}p(i_2,z_2|s)2^{|i_2|}\II_{\{z\equiv z_2\}}-1\right)\right]\nonumber\\
&=\EE_S\left[\sum_{(i_1,z_1)}\sum_{(i_2,z_2)}p(i_1,z_1|s)p(i_2,z_2|s)2^{-(N-|i_1|-|i_2|)}\left(\sum_{z_{-i_1}}\II_{\{(z_1,z_{-i_1})\equiv z_2\}}-1\right)\right]\nonumber\\
&=\EE_S\left[\EE_{(I_1,Z_1)|S\perp \! \! \! \perp(I_2,Z_2)|S}\left[2^{|I_1\cap I_2|}\II_{\{Z_1\equiv Z_2\}}-1\right]\right]
\end{align*}
\end{proof}

We note here that the above lemma is a special case (where $\bZ$ takes the uniform measure over $\{0,1\}^N$) of a more general lemma, which we state and prove in Appendix \ref{app:lbound}

\subsection{Sample-Complexity for Learning with $1$-bit Sketches}
\label{ssec:1bit}

To demonstrate the use of Lemma \ref{lemma:Ibnd}, we first consider a related problem that demonstrates the effect of per-user constraints (as opposed to average constraints) on the mutual information. We consider the same item-class learning problem as before with $w=1$ (i.e., each user has access to a single rating), but instead of a privacy constraint, we consider a `per-user bandwidth' constraint, wherein each user can communicate only \emph{a single bit} to the learning algorithm. 

\begin{theorem}
\label{Thm:one-bit}
Suppose $w=1$, with $(I,Z)$ drawn i.i.d uniformly over $[N]\times\{0,1\}$. Then for any $1$-bit sketch derived from $(I,Z)$, it holds that:
$\mut(\bZ,S) = O\left(\frac{1}{N}\right),$
and consequently, there exists a constant $c>0$ such that any cluster learning algorithm using queries with $1$-bit responses is unreliable if the number of users satisfies
$U < cN^2.$
\end{theorem}

\begin{proof}
In order to use Lemma~\ref{lemma:Ibnd}, we first note that $\mut(\bZ,S)$ is a convex function of $\PP[S=s|\bZ=z]$ for fixed $\PP[\bZ=\bz]$ (Theorem $2.7.4$ in~\cite{CoverThomas}). Writing $\PP[S=s|\bZ=\bz]$ as $\sum_{(i,z)} \PP[S=s|(I,Z)=(i,z)]\PP[(I,Z)=(i,z)|\bZ=\bz]$, we observe that the extremal points of the kernel $\PP[S=s|\bZ=\bz]$ correspond to $\PP[S=s|(i,z)] \in \{0,1\}$, where the mutual information is maximized. This implies that the class of deterministic queries with $1$-bit response that maximizes mutual information has the following structure: given user-data $(I_u,Z_u)$, the algorithm provides user $u$ with an arbitrary set $A\subseteq\{(i,z)|i\in[N],z\in\{0,1\}\}$ of (items,ratings), and the user identifies if $(I_u,Z_u)$ is contained in $A$. Formally, the query is denoted $S_u=\II_A(I_u,Z_u)$ (i.e., is $(I_u,Z_u)\in A?$).  

Defining $p_{i,z}^s:=\PP[(I,Z)=(i,z)|S=s]$, for a query response $S=\II_A(I_u,Z_u)$, we have the following:
\begin{align*}
p_{i,z}^1&=\frac{\PP[[(I,Z)=(i,z)]\PP[S=1|(i,z)]}{\sum_{(j,z_j')}\PP[(I,Z)=(j,z_j')]\PP[S=1|(j,z_j')]} \\
&=\frac{\II_A(i,z)}{\sum_{j=1}^N\left\{\II_A(j,0) + \II_A(j,1)\right\}}
= \frac{\II_A(i,z)}{|A|},
\end{align*}
and similarly $p_{i,z}^0=\frac{\II_{\bar{A}}(i,z)}{|\bar{A}|}$ where $\bar{A}$ is the complement of set $A$. From Lemma~\ref{lemma:Ibnd}, for r.v.s $(I_1,Z_1)\perp \! \! \! \perp(I_2,Z_2)|S$, we have:
\begin{align*}
\mut(\bZ,S)&\leq  \EE_{S}\left[\EE\left[2^{|I_1\cap I_2|}\II_{\{Z_1\equiv Z_2\}}-1\right]\right]\\ 
&=\sum_{s\in \{0,1\}}\PP[S=s]\EE\left[\II_{\{I_1= I_2\}}\left(2\II_{\{Z_1\equiv Z_2\}}-1\right)\right].
\end{align*}
Introducing the notation $\PP(I=\ell, Z(\ell)=\sigma|S=s)=\pi_{\ell,\sigma}^s$ , the following identity is easily established:
\begin{align}
\label{eq:3a}
\sum_{\ell=1}^N \EE\Big[ \II_{\{I_1=I_2=\ell\}}(2\II_{\{Z_1(\ell)=Z_2(\ell)\}}&-1)|S=s\Big]=
\sum_{\ell=1}^N\left(\pi_{\ell,0}^s-\pi_{\ell,1}^s\right)^2
\end{align}
The RHS of (\ref{eq:3a}) is a non-negative definite quadratic form of the variables $p_{i,z}^s$ (since $\pi_{\ell,\sigma}^s=\sum_{i,\sigma|\ell\in i,z(\ell)=\sigma}p_{i,z}^s$). Thus:
\begin{align*}
\mut(\bZ,S)&\leq\sum_{s\in \{0,1\}}\PP[S=s]\sum_{\ell=1}^N\left(\pi_{\ell,0}^s-\pi_{\ell,1}^s\right)^2\\
&= \sum_{s\in \{0,1\}}\PP[S=s] \frac{1}{|A_s|^2} \sum_{i=1}^N \II_{\{|A_s \cap \{(i,0),(i,1)\}|=1\}},
\end{align*}
where $A_s=A$ if $s=1$ and $\bar{A}$ if $s=0$. Now for a given $A$, consider the partitioning of the set $[N]$ into $C_0\cup C_1\cup C_2$, where for $k=1,2,3$, $\forall i\in C_k, |A \cap \{(i,0),(i,1)\}| = k$. We then have the following:
\begin{align*}
\mut(\bZ,S) & \leq\PP[S=1]\frac{|C_1|}{|A|^2} + \PP[S=0]\frac{|C_1|}{|\bar{A}|^2}\\
& = \frac{|A|}{2N}\frac{|C_1|}{|A|^2} + \frac{|\bar{A}|}{2N}\frac{|C_1|}{|\bar{A}|^2}\mbox{\hspace{1cm}(Since $S=\II_A(I,Z)$)}\\
& = \frac{|C_1|}{2N}\left(\frac{1}{|A|} + \frac{1}{2N-|A|}\right)
\leq \frac{1}{N}.
\end{align*}
Now, using Fano's inequality (Lemma~\ref{lemma:fano}) to get the result.
\end{proof}

Note that the above bound is tight -- to see this, consider a (adaptive) scheme where each user is asked a random query of the form ``Is $(I_u,Z_u)=(i,b)$?''(where $i\in[N]$ and $b=\{0,1\}$). The average time between two successful queries is $2N$, and one needs $N$ successful queries to learn all the bits. This demonstrates an interesting change in the sample-complexity of learning with per-user communications constraints ($1$-bit sketches in this section, privacy in next section) versus average-user constraints (mutual information bound or average bandwidth). 

\subsection{Sample-Complexity for Learning under Local-DP}
\label{ssec:LB}

We now exploit the above techniques to obtain lower bounds on the scaling required for accurate clustering with DP in an information-scarce regime, i.e., when $w=o(N)$. To do so, we first require a technical lemma that establishes a relation between the distribution of a random variable with and without conditioning on a differentially private sketch:
\begin{lemma}
\label{lemma:uncondition}
Given a discrete random variable $A\in\mathcal{A}$ and some $\epsilon$-differentially private `sketch' variable $S\in\mathcal{S}$ generated from $A$, there exists a function $\lambda:\mathcal{A}\times\mathcal{S}\rightarrow[e^{-\epsilon},e^{\epsilon}]$ such that for any $a\in\mathcal{A}$ and $s\in\mathcal{S}$:
\begin{equation}
\PP(A=a|S=s)=\PP(A=a) \lambda(a,s)
\end{equation}
\end{lemma}

\begin{proof}
\begin{align*}
\PP(A=a|S=s)&=\frac{\PP(A=a)\PP(S=s|A=a)}{\sum_{a'\in\mathcal{A}}\PP(A=a')\PP(S=s|A=a')}\\
&\quad\mbox{(From Bayes' Theorem)}\\
&=\PP(A=a)\left(\sum_{a'\in\mathcal{A}}\PP(A=a')\frac{\PP(S=s|A=a')}{\PP(S=s|A=a)}\right)^{-1}
\end{align*}
Thus, we can define: $$\lambda(a,s)=\left(\sum_{a'\in\mathcal{A}}\PP(A=a')\frac{\PP(S=s|A=a')}{\PP(S=s|A=a)}\right)^{-1}.$$
Further, from the definition of $\epsilon$-DP, we have:
\begin{equation*}
e^{-\epsilon}\leq\frac{\PP(S=s|A=a')}{\PP(S=s|A=a)}\leq e^{\epsilon},
\end{equation*}
and hence we have $\lambda(a,s)\in [e^{-\epsilon},e^{\epsilon}],\,\forall\,a\in\mathcal{A},s\in\mathcal{S}$.
\end{proof}

Recall we define $\mathcal{D}:=[N]_w\times \{0,1\}^w$ to be the set from which user information $(I,Z)$ is drawn. We write $\PP^0$ for the base probability distribution on $(I_1,Z_1)$ and $(I_2,Z_2)$ (note: the two are i.i.d uniform) over $\mathcal{D}$, and denote by $\EE^0$ mathematical expectation under $\PP^0$. We also need the following estimate (c.f. Appendix \ref{app:lbound} for the proof):
\begin{lemma}
\label{lemma:combbnd}
If $w=o(N)$, then:
\begin{equation*}
\left|\frac{\binom{N-w}{w}}{\binom{N}{w}}-\left(1-\frac{w^2}{N}\right)\right|=\Theta\left(\frac{w^4}{N^2}\right)
\end{equation*}
\end{lemma}

\noindent We can prove our tightened bounds. We first obtain a weak lower bound in Theorem~\ref{Thm:weaklb}, valid for all $w$, and then refine it in Theorem~\ref{Thm:stronglb} under additional conditions. 

\begin{theorem}
\label{Thm:weaklb} 
In the information-scarce regime, i.e., when $w=o(N)$, under $\epsilon$-local-DP we have:
\begin{eqnarray*}
\mut(\bZ,S)=O\left(\frac{w^2}{N}\right)
\end{eqnarray*}
and consequently, there exists a constant $c>0$ such that any cluster learning algorithm with $\epsilon$-local-DP is unreliable if the number of users satisfies  
$U < c\left(\frac{N^2}{w^2}\right).$
\end{theorem}

\begin{proof}
To bound the mutual information between the underlying model and each private sketch, we use Lemma~\ref{lemma:Ibnd}. In particular, we show that the mutual information is bounded by $\left(\frac{w^2}{N}\right)$ for any given value $s$ of the private sketch. 

\noindent Consider any sketch realization $S=s$.  Now, we have:
\begin{align*}
\EE\left[2^{|I_1\cap I_2|}\II_{\{Z_1\equiv Z_2\}}-1\right]
&\leq\EE\left[\II_{\{Z_1\equiv Z_2\}}\left(2^{|I_1\cap I_2|}-1\right)\right]
\end{align*}
The RHS of the above equation is a non-negative quadratic function of the variables $\{p_{i,z}\}_{(i,z)\in\mathcal{D}}$, where $p_{i,z}:=\PP[(I,Z)=(i,z)|S=s]\}$. Now, using Lemma~\ref{lemma:uncondition}, we get:
\begin{align*}
\EE\left[2^{|I_1\cap I_2|}\II_{\{Z_1\equiv Z_2\}}-1\right]
&\leq e^{2\epsilon}\EE^0\left[\II_{\{Z_1\equiv Z_2\}}\left(2^{|I_1\cap I_2|}-1\right)\right]\\
&=e^{2\epsilon}\sum_{k=0}^{w}\EE^0\left[\II_{\{|I_1\cap I_2|=k\}}\II_{\{Z_1\equiv Z_2\}}\left(2^{|I_1\cap I_2|}-1\right)\right]\\
&=e^{2\epsilon}\sum_{k=0}^{w}\EE^0\left[\II_{\{|I_1\cap I_2|=k\}}2^{-k}\left(2^{k}-1\right)\right]\\
&=e^{2\epsilon}(\Delta_1+\Delta_2),
\end{align*}
where we define:
\begin{align*}
\Delta_1&=\frac{1}{2}\EE^0\left[\II_{\{|I_1\cap I_2|=1\}}\right],\\ 
\Delta_2&=\EE^0\left[\II_{|I_1\cap I_2|>1}\left(1-2^{-|I_1\cap I_2|}\right)\right]
\end{align*}
Now we bound each of these terms separately. For $\Delta_1$:
\begin{align}
\Delta_1&=\frac{1}{2}\EE^0\left[\II_{\{|I_1\cap I_2|=1\}}\right]\nonumber\\
&=\frac{1}{2}\sum_{\ell=1}^N\EE^0\left[\II_{\{I_1\cap I_2=\{\ell\}\}}\right]=\frac{w\binom{N-w}{w-1}}{2\binom{N}{w}}\nonumber\\
&=\frac{w^2}{2(N-2w+1)}\left(1-\frac{w^2}{N}+O\left(\frac{w^4}{N^2}\right)\right)\nonumber\\
&\mbox{(Using Lemma~\ref{lemma:combbnd})}\nonumber\\
&=\frac{w^2}{2(N-2w+1)}\left(1-\frac{w^2}{N}+O\left(\frac{w^4}{N^2}\right)\right)\nonumber\\
&=O\left(\frac{w^2}{N}\right)\label{eq:delta1bnd}
\end{align}
Similarly for $\Delta_2$, we have:
\begin{align}
\Delta_2 &\leq\EE^0\left[\II_{\{|I_1\cap I_2|>1\}}\right]
=1-\PP^0\left[|I_1\cap I_2|<2\right]\nonumber\\
&=1-\frac{\binom{N-w}{w}+w\binom{N-w}{w-1}}{\binom{N}{w}}\nonumber\\
&=1-\left(1+\frac{w^2}{N-2w+1}\right)\frac{\binom{N-w}{w}}{\binom{N}{w}}\nonumber\\
&=1-\left(1+\frac{w^2}{N-2w+1}\right)\left(1-\frac{w^2}{N}-O\left(\frac{w^4}{N^2}\right)\right)\nonumber\\
&=O\left(\frac{w^4}{N^2}\right)\cdot\label{eq:delta2bnd}
\end{align}
Combining equations (\ref{eq:delta1bnd}) and (\ref{eq:delta2bnd}), we get the result.
\end{proof}

The above result shows how Lemma~\ref{lemma:Ibnd} can be used to obtain sharper bounds on the mutual information contained in a differentially private sketch in the information-scarce setting in comparison to Lemma~\ref{lemma:basicIbnd}. Theorem \ref{Thm:weaklb} gives a lower bound of $\Omega(\frac{N^2}{w^2})$ on the number of samples needed to learn the underlying clustering. Observe however that the dominant term in the above proof is the bound on $\Delta_1$ -- a more careful analysis of this leads to the following stronger bound:
\begin{theorem}
\label{Thm:stronglb} 
Under the scaling assumption $w=o(N^{1/3})$, and for $\epsilon<\ln(2)$, it holds that
\begin{equation*}
\mut(\bZ,S)=O\left(\frac{w}{N}\right)\cdot
\end{equation*}
and thus there exists a constant $c>0$ such that any cluster learning algorithm with $\epsilon$-local-DP is unreliable if the number of users satisfies 
$U < c\left(\frac{N^2}{w}\right).$
\end{theorem}

The proof of Theorem \ref{Thm:stronglb} is much more technical than that of Theorem \ref{Thm:weaklb} -- we provide an outline below, and defer the complete proof to Appendix \ref{app:lbound}. 
\begin{proof}[Proof Outline]
Starting from Lemma~\ref{lemma:Ibnd}, we first perform a decomposition of the bound.  For any $S=s$, we establish:
\begin{align}
\label{eq:withDelta}
EE\left[2^{|I_1\cap I_2|}\II_{\{Z_1\equiv Z_2\}}-1\right] \leq \sum_{\ell=1}^N\EE\left[\II_{\{\ell \in I_1\cap I_2\}}(2*\II_{Z_1(\ell)=Z_2(\ell)}-1)\right]+O\left(\frac{w^4}{N^2}\right)
\end{align}
Under the scaling assumption $w=o(N^{1/3})$, the second term in the right-hand side of the above equation is $o(w/N)$, and we only need to establish that the first term in the right-hand side is $O(w/N)$. Using the notation $\PP(\ell \in I, Z(\ell)=\sigma|S=s)=\pi_{\ell,\sigma}$, we establish the following:
\begin{equation*}
\sum_{\ell=1}^N \EE \II_{\{\ell\in I_1\cap I_2\}}\left(2\II_{Z_1(\ell)=Z_2(\ell)}-1\right)= \sum_{\ell=1}^N\left(\pi_{\ell,0}-\pi_{\ell,1}\right)^2.
\end{equation*}
Now defining $p_{i,z}:=\PP(I=i,Z=z|S=s)$ we have $\pi_{\ell,\sigma}=\sum_{i,\sigma|\ell\in i,z(\ell)=\sigma}p_{i,z}$.  
We formulate the problem of upper-bounding the first term on the right-hand side of~(\ref{eq:withDelta}) as the following optimization problem:
\begin{equation*}
\begin{aligned}
& \underset{\{p_{i,z}\}_{(i,z)\in\mathcal{D}}}{\text{Maximize}}
&& \sum_{\ell=1}^N\left(\pi_{\ell,0}-\pi_{\ell,1}\right)^2 \\
& \text{Subject to}
&& \sum_{(i,z)\in\mathcal{D}}p_{i,z} = 1, \quad 
p_{i,z}D\in\left[1-\epsilon',1+\epsilon'\right].
\end{aligned}
\end{equation*}
Where $\epsilon'=\max(e^{\epsilon}-1,1-e^{-\epsilon})$, and the constraint is derived from the $\epsilon$-DP definition.
 
We first establish that the extremal points of the above convex set consist of the distributions $p^A_{i,z}$ indexed by the sets $A\subset \mathcal{D}$ of cardinality $D/2$, defined by $p^A_{i,z} = \frac{1+\epsilon'}{D}$ if $(i,z) \in A$ and $\frac{1-\epsilon'}{D}$ otherwise.  We then show that for each such $A$, $\sum_{\ell=1}^N\left(\pi^A_{\ell,0}-\pi^A_{\ell,1}\right)^2\le O(w/N)$,
where $\pi^A_{\ell,\sigma}=\sum_{i: \ell \in i}\sum_{z:z(\ell)=\sigma}p^A_{i,z}$.
\end{proof}

Significantly, however, the bound in Theorem \ref{Thm:stronglb} matches the performance of the MaxSense algorithm, which we present next, thereby showing that it is tight.

\section{Local-DP in the Information-Scarce Regime: The MaxSense Algorithm}
\label{sec:maxsense}

The Pairwise-Preference algorithm of Section~\ref{sec:rich}, although orderwise optimal in the information-rich regime, is highly suboptimal in the information-scarce setting. In particular, note that the probability that two randomly probed items have been rated by the user is $O(w^2/N^2)$ -- now, in order to obtain the same guarantees as in the information-rich regime (where we needed $\Theta(N\log(N))$ samples to learn the cluster labels), we now need $\Theta(N\log(N)\cdot N^2/w^2)$ users -- this however is polynomially larger than our lower bounds from Section~\ref{sec:LBscarce}. 

This suggests that in order to learn in an information-scarce regime, an algorithm needs to probe or `sense' a much larger set of items (intuitively, of the order of $\frac{N}{w}$) in order to hit the set of watched items with a non-vanishing probability. We now outline the MaxSense algorithm for cluster-learning in the information-scarce regime, which is based on this intuition. 

As with Pairwise Preference, MaxSense uses (privatized) $1$-bit sketches for learning -- however each sketch now aggregates ratings for several items. A query to user $u$ is formed by constructing a {\em random sensing vector} $H_u=(H_{ui})_{i\in[N]}$, whose entries $H_{ui}=1$ if item $i$ is being sensed, and $0$ otherwise. Each item $i$ is chosen for sensing (i.e. $H_{ui}$ is set to $1$) in an i.i.d. manner with probability $\theta/w$ (for some chosen constant $\theta>0$). User $u$ then constructs a private sketch $S^0_u$, which is the maximum of her ratings for items that are being sensed; as before, unrated items are given a rating of $0$. Formally, $S^0_u=\max_{i\in [N]}H_{ui}Z_{ui}$, where $Z_{ui}$ is $1$ if user $u$ rated item $i$ positively, else $0$. Finally, user $u$ outputs a privatized version $S_u$ of $S^0_u$. The sensing vector $H_u$ is assumed to be known publicly.

Based on the sketches $S_u$ and sensing vectors $H_u$, the algorithm then determines a per-item score given by $B_i=\sum_{u\in[U]}H_{un}S_u,\; n\in [N]$. Finally, it performs $k$-means clustering of these scores in $\setR$. The algorithm is formally specified in Figure \ref{alg:maxsense}.

\begin{figure}[!hb]
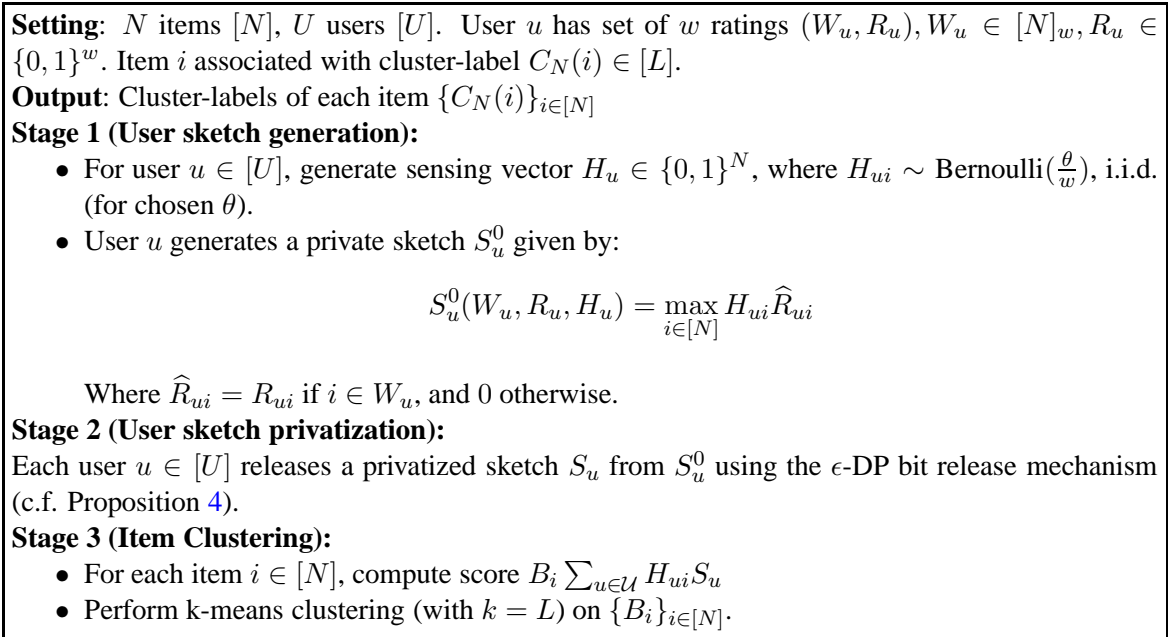

\textbf{Setting}: $N$ items $[N]$, $U$ users $[U]$. User $u$ has set of $w$ ratings $(W_u,R_u), W_u\in[N]_w, R_u\in\{0,1\}^w$. Item $i$ associated with cluster-label $C_N(i)\in[L]$.

\textbf{Output}: Cluster-labels of each item $\{C_N(i)\}_{i\in[N]}$\\
\textbf{Stage 1 (User sketch generation):} 
\begin{itemize}[nolistsep,noitemsep]
\item For user $u\in[U]$, generate sensing vector $H_u\in\{0,1\}^N$, where $H_{ui}\sim\mbox{Bernoulli}(\frac{\theta}{w})$, i.i.d. (for chosen $\theta$).
\item User $u$ generates a private sketch $S_u^0$ given by:
\begin{equation*}
S_u^0(W_u,R_u,H_u)=\max_{i\in[N]}H_{ui}\widehat{R}_{ui}
\end{equation*}
Where $\widehat{R}_{ui}=R_{ui}$ if $i\in W_u$, and $0$ otherwise.
\end{itemize}

\textbf{Stage 2 (User sketch privatization):}\\ 
Each user $u\in[U]$ releases a privatized sketch $S_u$ from $S_u^0$ using the $\epsilon$-DP bit release mechanism (c.f. Proposition \ref{prop:1bitpriv}).

\textbf{Stage 3 (Item Clustering):}
\begin{itemize}[nolistsep,noitemsep]
\item For each item $i\in[N]$, compute score $B_i\sum_{u\in\mathcal{U}}H_{ui}S_u$
\item Perform k-means clustering (with $k=L$) on $\{B_i\}_{i\in[N]}$. 
\end{itemize}
\caption{The MaxSense Algorithm}
\label{alg:maxsense}
\end{figure}

\begin{theorem}
\label{Thm:maxsense}
The MaxSense algorithm satisfies $\epsilon$-local-DP. Further, let $\widehat{\epsilon} =\frac{2(e^{\epsilon}-1)}{(e^{\epsilon}+1)}$, and define: 
\begin{align*}
	\delta_{\min} =\min_{1\leq \ell< \ell'\leq L}\left|\sum_{k=1}^{K}\alpha_ke^{-\theta\sum_{\ell=1}^L\beta_{\ell}b_{k\ell}}(b_{k\ell}-b_{k\ell'})\right|
\end{align*}
(where the item sensing probability is $\theta/w$). Then for any $d>0$, there exists a constant $C>0$ such that the clustering is successful with  probability $1-N^{-d}$ if the number of users satisfies:
\begin{equation}
\label{eq:Uscale}
U\geq C\left(\frac{N^2\log N}{\widehat{\epsilon}^2\delta_{\min}^2w}\right).
\end{equation}
\end{theorem}

\noindent Before presenting the proof, we note that $\delta_{\min}$ encodes the required separability conditions for successful clustering. In particular, let $v_k=\sum_{\ell} \beta_{\ell}b_{k\ell}$ -- then it can be checked that $\delta_{min}$ is strictly positive for all $\theta$ (except on a set of measure $0$) provided the following holds:
\begin{equation*}
\forall\ell\ne \ell'\in[L], \exists k\in[K] \hbox{ such that }\sum_{j:v_j=v_k} \alpha_j(b_{j\ell}-b_{j\ell'})\ne 0.
\end{equation*} 
Designing algorithms with similar performance under weaker separability remains an open problem.

\begin{proof}
We use $k(u)$ to denote the user-cluster of user $u$ and $l(j)$ to denote the cluster of item $j$.

\noindent\emph{Privacy:} For each user $u$, note that $H_u$ is independent of the data $(W_u,R_u)$. Next, given $H_u$, we have that $(W_u,R_u)\rightarrow S_u^0 \rightarrow S_u$ form a Markov chain, and hence it is sufficient (via the post-processing property) to prove that $S_u^0\rightarrow S_u$ satisfy $\epsilon$-differential privacy. This however is a direct consequence of using the $\epsilon$-DP bit release mechanism. 

\noindent\emph{Performance:} An overview of the proof of correctness of MaxSense is as follows: First, we show that for any item $j$, its count $B_j$ concentrates around $\overline{B_{l(j)}}$, the expected count for its corresponding cluster. Next, we compute the minimum separation $\Delta_{\min}$ between the expected counts for any two item-clusters. Finally, we show that under the given scaling of users, each item count $B_j$ is within a distance $\Delta_{\min}/5$ from $\overline{B_l(j)}$ w.h.p. This implies that any two items belonging to the same cluster are within a distance of $2\Delta_{\min}/5$, while two items of different clusters have a separation of at least $3\Delta_{\min}/5$, thereby ensuring successful clustering.  

First, let $p=\frac{\theta}{w}$ denote the sensing probability, and define:
\begin{equation*}
q_k^0:=\PP[S_u^0=0|k(u)=k]=\prod_{j\in[N]}\left(1-\frac{pw b_{kl(j)}}{N}\right),
\end{equation*}
i.e., $q_k^0$ is the probability that a user $u$ of cluster $k$ will have a (private) sketch $S_u^0$ equal to $0$. Then we have:
\begin{align}
\log q_k^0&=\sum_{j=1}^N\log\left(1-\frac{\theta b_{kj}}{N}\right)
=\sum_{l=1}^L\beta_lN\log\left(1-\frac{\theta b_{kl}}{N}\right)\nonumber\\
&=\sum_{l=1}^L\beta_lN\left(-\frac{\theta b_{kl}}{N}+\Theta\left(\frac{1}{N^2}\right)\right)\nonumber\\
&=-\theta\sum_{l=1}^L\beta_lb_{kl}+\Theta\left(\frac{1}{N}\right)
\label{eq:qk0}
\end{align}
Thus $\log q_k^0\geq -\theta+\Theta\left(\frac{1}{N}\right)$, from which we have:
\begin{equation*}
\frac{1}{e^\theta}\left(1+\Theta\left(\frac{1}{N}\right)\right)\leq q_k^0\leq 1,
\end{equation*}
\noindent Thus we see that for any user, the probability of the MaxSense sketch being $0$ is $\Theta(1)$. Intuitively, this means that each sketch has $>0$ bits of information. We define $q^0=\sum_{k=1}^Kq_k^0$ (i.e., the probability that a random user's sketch is $0$).

Next, for any item $i\in[N]$, consider the item-score $B_i=\sum_{u\in\mathcal{U}}H_{ui}S_i$. From the i.i.d sensing property and the $\epsilon$-DP bit release mechanism mechanism, we have:
\begin{align*}
\mathbb{E}[B_i]&=\sum_{u\in\mathcal{U}}\mathbb{E}[H_{ui}S_u]=\sum_{u=1}^{U}p\left[\frac{\mathbb{E}[1-S_u^0|H_{ui}=1]}{e^{\epsilon}+1}+\frac{e^{\epsilon}\mathbb{E}[S_u^0|H_{ui}=1]}{(e^{\epsilon}+1)}\right],
\end{align*}
Substituting $\widehat{\epsilon}=\frac{2(e^{\epsilon}-1)}{(e^{\epsilon}+1)}$, we can expand the expression for $\mathbb{E}[B_i]$ as follows:
\begin{align*}
\mathbb{E}[B_i]&=\sum_{u=1}^Up\left[\frac{1}{2}-\frac{\widehat{\epsilon}}{4}+\frac{\widehat{\epsilon}}{2}\mathbb{E}\left[\left.\mathbb{E}[S_u^0|H_u]\right|H_{ui}=1\right]\right]\\
&=\sum_{u=1}^Up\left[\frac{1}{2}-\frac{\widehat{\epsilon}}{4}+\frac{\widehat{\epsilon}}{2}\mathbb{E}\left[\left.1-\prod_{j\in[N]}\left(1-\frac{w b_{uj}H_{uj}}{N}\right)\right|H_{ui}=1\right]\right]\\
&=\sum_{u=1}^Up\left[\frac{1}{2}+\frac{\widehat{\epsilon}}{4}-\frac{\widehat{\epsilon}}{2}\left(1-\frac{w b_{ui}}{N}\right)\mathbb{E}\left[\prod_{j\neq i}\left(1-\frac{w b_{uj}H_{uj}}{N}\right)\right]\right]\\
&=\sum_{u=1}^Up\left[\frac{1}{2}+\frac{\widehat{\epsilon}}{4}-\frac{\widehat{\epsilon}}{2}\left(1-\frac{w b_{k(u)l(i)}}{N}\right)\prod_{j\neq i}\left(1-p+p\left(1-\frac{w b_{k(u)l(j)}}{N}\right)\right)\right]\\
&\mbox{(Using the i.i.d sensing properties of $H_{ui}$)}\\
&=\sum_{k=1}^K\alpha_kUp\left[\frac{1}{2}+\frac{\widehat{\epsilon}}{4}-\frac{\widehat{\epsilon}}{2}\left(1-\frac{w b_{kl(j)}}{N}\right)\left(1-\frac{pw b_{kl(i)}}{N}\right)^{-1}\prod_{j\in[N]}\left(1-\frac{pw b_{kl(j)}}{N}\right)\right]\\
&\mbox{(Grouping terms by user and item classes.)}
\end{align*}
Note that we have dropped the explicit dependence on the user index and retained only the user-cluster label. Similarly, we henceforth write $k$ and $l$ for $k(i), l(j)$ respectively, whenever it does not cause confusion in the notation. Thus we have:
\begin{align*}
\mathbb{E}[B_i]&=Up\left[\frac{1}{2}+\frac{\widehat{\epsilon}}{4}-\frac{\widehat{\epsilon}}{2}\sum_{k=1}^K\alpha_kq_k^0\left(1-\frac{w b_{kl(i)}}{N}\right)\left(1-\frac{pwb_{kl(i)}}{N}\right)^{-1}\right]\\
&=Up\left[\frac{1}{2}+\frac{\widehat{\epsilon}}{4}-\frac{\widehat{\epsilon}}{2}\sum_{k=1}^K\alpha_kq_k^0\left(1-\frac{\frac{w(1-p)}{N}b_{kl(i)}}{\left(1-\frac{pwb_{kl(i)}}{N}\right)}\right)\right]\\
&= Up\left[\frac{1}{2}+\frac{\widehat{\epsilon}}{4}-\frac{\widehat{\epsilon}}{2}\sum_{k=1}^K\alpha_kq_k^0\right] +Up\frac{(w-\theta)}{N}\frac{\widehat{\epsilon}}{2}\sum_{k=1}^K\frac{\alpha_kq_k^0b_{kl(i)}}{\left(1-\frac{\theta}{N}b_{kl(i)}\right)}
\end{align*}
Now, noting that $\EE[B_i]$ only depends on the class $l(i)$ of item $i$, we define $\overline{B_l}=\mathbb{E}[B_i|l(i)=l]$. Then we have:
\begin{equation*}
\overline{B_l}=Up\left[\frac{1}{2}+\frac{\widehat{\epsilon}}{4}-\frac{\widehat{\epsilon}}{2}q^0\right]
+Up\frac{(w-\theta)}{N}\frac{\widehat{\epsilon}}{2}\sum_{k=1}^K\frac{\alpha_kq_k^0b_{kl}}{\left(1-\frac{\theta b_{kl}}{N}\right)}\\
\end{equation*}
Recall $w=o(N)$, and $\widehat{\epsilon}<1$ -- hence, for sufficiently large $N$, we have that for all item classes $l\in[L]$:$\quad\overline{B_l}\leq Up$.

Next, given any two distinct item classes $l,m$, we define $\Delta_{lm}:=\mathbb{E}[|B_l-B_m|]$. Then we have:
\begin{align*}
\Delta_{lm}&\geq |\mathbb{E}[B_l-B_m]|\quad\mbox{(By Jensen's Inequality)}\\
 &= Up\frac{(w-\theta)}{N}\frac{\widehat{\epsilon}}{2}\left|\sum_{k=1}^K\frac{\alpha_kq_k^0b_{kl}}{\left(1-\frac{\theta b_{kl}}{N}\right)}-\sum_{k=1}^K\frac{\alpha_kq_k^0b_{km}}{\left(1-\frac{\theta b_{km}}{N}\right)}\right|\\
&\geq\frac{U\widehat{\epsilon}(c-c^2w^{-1})}{N}\delta_{lm},
\end{align*}
\noindent where we define (using equation \ref{eq:qk0}):
\begin{align*}
\delta_{lm}&:=\left|\sum_{k=1}^K\frac{\alpha_kq_k^0b_{kl}}{\left(1-\frac{\theta b_{kl}}{N}\right)}-\sum_{k=1}^K\frac{\alpha_kq_k^0b_{km}}{\left(1-\frac{\theta b_{km}}{N}\right)}\right|\\
&=\left|\sum_{k=1}^K\alpha_kq_k^0\frac{(b_{kl}-b_{km})}{\left(1-\frac{\theta b_{kl}}{N}\right)\left(1-\frac{\theta b_{km}}{N}\right)}\right|\\
&\geq\left|\sum_{k=1}^K\alpha_k(b_{kl}-b_{km})e^{-\theta\sum_{l=1}^L\beta_lb_{kl}}\right|\\
&\geq\delta_{\min}:=\min_{1\leq l<l'\leq L}\left|\sum_{k=1}^{K}\alpha_k(b_{kl}-b_{kl'})e^{-\theta\sum_{l=1}^L\beta_lb_{kl}}\right|.
\end{align*}
Let $\Delta_{\min}:=\min_{l,m\in[L]^2,l\neq m}\Delta_{lm}$. Now, for a given item $j$, a standard Chernoff bound (applicable since the sketches are independent and bounded) gives us that for any $a>0$:
\begin{equation*}
\mathbb{P}[|B_j-\overline{B_l(j)}|\geq a\overline{B_l(j)}]\leq 2\exp\left(-\frac{a^2}{3}\overline{B_l(j)}\right)
\end{equation*}
Choose $a=\frac{\Delta_{\min}}{5\overline{B_l(j)}}$. Then we have:
\begin{align*}
\mathbb{P}\left[|B_j-\overline{B_l(j)}|\geq \frac{\Delta_{\min}}{5}\right]&\leq 2\exp\left(-\frac{\Delta_{\min}^2}{75\overline{B_l(j)}}\right)\leq 2\exp\left(-\left(\frac{\Delta_{\min}^2}{75Up}\right)\right),
\end{align*}
and by taking union bound over all items, we have:
\begin{align*}
\mathbb{P}\Bigg[\sup_{j\in[N]}|B_j-\overline{B_l(j)}|\geq \frac{\Delta_{\min}}{5}\Bigg]
&\leq \exp\left(\log 2N-\frac{\Delta_{\min}^2}{75Up}\right)
\leq \exp\left(\log 2N-\frac{Uw\widehat{\epsilon}^2\delta_{\min}^2}{75N^2}\right),
\end{align*}
where we have substituted $p=\frac{c}{w}$. Now if we choose $U$ as:
\begin{equation*}
U=\left(\frac{75N^2(\log 2+(1+d)\log N)}{\widehat{\epsilon}^2\delta_{\min}^2w c}\right)=\Theta\left(\frac{N^2\log N}{\widehat{\epsilon}^2\delta_{\min}^2w c}\right),
\end{equation*}
then we have:
\begin{equation*}
\mathbb{P}\left[\sup_{j\in[N]}|B_j-\overline{B_l(j)}|\geq \frac{\Delta_{\min}}{5}\right]\leq \frac{1}{N^d},
\end{equation*}
Thus, if number of users scale as in (\ref{eq:Uscale}), then clustering is successful with probability $1-N^{-d}$.
\end{proof}

Theorem~\ref{Thm:maxsense} demonstrates that  MaxSense is sufficient to achieve optimal scaling in $N$ (up to logarithmic terms) under suitable separability condition. One problem is that MaxSense does not achieve the optimal `privacy trade-off', namely, a $\frac{1}{\epsilon}$ factor in required sample-complexity scaling. To correct this, we propose the \emph{Multi-MaxSense} algorithm, a generalization of Algorithm \ref{alg:maxsense}, wherein we ask multiple MaxSense queries to each user. 

In Multi-MaxSense, each \emph{query} now has an associated privacy parameter of $\frac{\epsilon}{Q}$, where $Q$ is the number of questions asked to a user -- thus, for each user, we still maintain $\epsilon$-local-DP via the composition property (Proposition \ref{prop:composition}). Independence between the answers is ensured as follows: first, for each user, we choose a random partition of $[N]$ into $\frac{1}{p}$ sets, each of size $Np$; we pick $Q$ of these and present them to the user. Next, each user calculates $Q$ sketches using these $Q$ sensing vectors, and reveals the privatized set of sketches (with each sketch revelation obeying $\frac{\epsilon}{Q}$-differential privacy. Finally, we compute and cluster the item-counts as before. Formally, the algorithm is specified in Figure \ref{alg:mmaxsense}. Now we have the following theorem.

\begin{figure}[!htb]
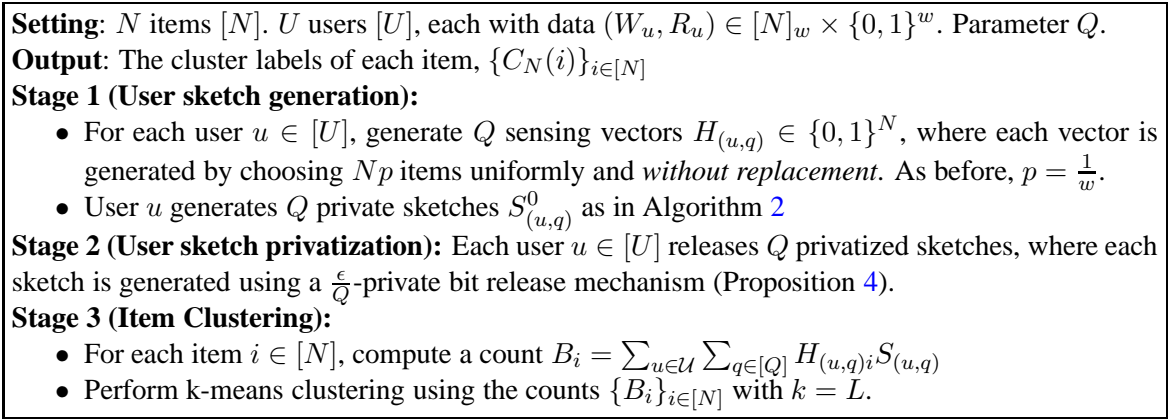

\textbf{Setting}: $N$ items $[N]$. $U$ users $[U]$, each with data $(W_u,R_u)\in[N]_w\times\{0,1\}^w$. Parameter $Q$.

\textbf{Output}: The cluster labels of each item, $\{C_N(i)\}_{i\in[N]}$

\textbf{Stage 1 (User sketch generation):} 
\begin{itemize}[nolistsep,noitemsep]
\item For each user $u\in[U]$, generate $Q$ sensing vectors $H_{(u,q)}\in\{0,1\}^N$, where each vector is generated by choosing $Np$ items uniformly and \emph{without replacement}. As before, $p=\frac{1}{w}$.
\item User $u$ generates $Q$ private sketches $S_{(u,q)}^0$ as in Algorithm \ref{alg:maxsense}
\end{itemize}

\textbf{Stage 2 (User sketch privatization):} Each user $u\in[U]$ releases $Q$ privatized sketches, where each sketch is generated using a $\frac{\epsilon}{Q}$-private bit release mechanism (Proposition \ref{prop:1bitpriv}).

\textbf{Stage 3 (Item Clustering):}
\begin{itemize}[nolistsep,noitemsep]
\item For each item $i\in[N]$, compute a count $B_i=\sum_{u\in\mathcal{U}}\sum_{q\in [Q]}H_{(u,q)i}S_{(u,q)}$
\item Perform k-means clustering using the counts $\{B_i\}_{i\in[N]}$ with $k=L$. 
\end{itemize}
\caption{The Multi-MaxSense Algorithm}
\label{alg:mmaxsense}
\end{figure}

\begin{theorem}
\label{Thm:mmaxsense}
The Multi-MaxSense algorithm satisfies $\epsilon$-local-DP. Further, suppose $Q=\lceil\epsilon\rceil$. Then for any $d>0$, there exists a constant $c$ such that if the number of users satisfies:
\begin{equation*}
U\geq c\left(\frac{N^2\log N}{\epsilon\delta_{\min}^2w}\right),
\end{equation*}
then the clustering is successful with probability $1-N^{-d}$.
\end{theorem}
\begin{proof}\\
\emph{Privacy:} Since each user reveals $Q$ bits, and each bit is privatized using a $\frac{\epsilon}{Q}$-DP mechanism, therefore for any user $u$, the $Q$ user sketches $\{S_{u,q}\}_{q=1}^{Q}$ and user data $(W_u,R_u)$ satisfy $\epsilon$-DP using the composition property (Proposition \ref{prop:composition}). The remaining proof for the privacy of the learning algorithm is as before, using the post-processing property.

\noindent\emph{Performance:} To show the improved scaling, observe that:
\begin{enumerate}[nolistsep,noitemsep]
\item Due to choice of sensing vectors, the probability of any probe for a item in any sensing vector (i.e., $H_{(u,q)i}$ for some $u\in[U],q\in[Q],i\in[N]$) being set to $1$ is $p$, i.i.d. 
\item Further, since the multiple sensing vectors given to a single user do not overlap, therefore the sketches $\{S_{u,q}\}_{u,q}$ are also independent.
\end{enumerate}
Hence, the analysis in Algorithm \ref{alg:maxsense} can be repeated with $U$ being replaced with $QU$ and $\epsilon$ being replaced with $\frac{\epsilon}{Q}$. Choosing $Q=\lceil\epsilon\rceil$  implies that we now have:
\begin{align*}
\widehat{\epsilon}&=\frac{2\left(\exp\left(\frac{\epsilon}{\lceil\epsilon\rceil}\right)-1\right)}{\exp\left(\frac{\epsilon}{\lceil\epsilon\rceil}\right)+1}\geq\frac{2(e-1)}{e+2}
\end{align*}
Substituting these in equation \ref{eq:Uscale}, we get the condition for correct clustering w.h.p as:
\begin{align*}
U &\geq c'\left(\frac{N^2\log N}{\epsilon\delta_{\min}^2w}\right).
\end{align*}
\end{proof}

\section{Lower Bounds under Adaptive Queries}
\label{ssec:adaptive}

The lower bounds of Section~\ref{sec:LBscarce} applied to non-adaptive learning, where queries to users are performed in parallel, without leveraging answers of users $1,\ldots,u-1$ when querying user $u$. We now extend these bounds to the adaptive setting, where we now assume that users are queried sequentially, and the query for the $t$-th user can be affected by the sketches $S_1^{t-a}:=\{S_1,\ldots,S_{t-1}\}$ released by the $t-1$ previous users. We now have the following sample-complexity lower bound:

\begin{theorem}
\label{thm:adaptive}
Assume $w=1$, and items are uniformly clustered into one of two clusters $\{0,1\}$ (i.e., $C_N(\cdot)$ is drawn uniformly at random from $\{0,1\}^N$). If users' responses satisfy $\epsilon$-local-DP, then the number of {\em adaptive} queries needed to learn the clustering $C_N(\cdot)$ is $\Omega(N\log N)$.
\end{theorem}
\noindent To prove this, we first need a generalization of Lemma~\ref{lemma:Ibnd}. As the proof is similar to Lemma~\ref{lemma:Ibnd}, we defer it to Appendix \ref{app:lbound}.
\begin{lemma} 
\label{lemma:Ibndgen}
Assume that under measure $\PP$, the set $I$ of items sampled by a user is independent of the type vector $Z$. Let $p_s(i,z):=\PP((I,Z_I)=(i,z)|S=s)$, and for any subset $j\subset[N]$ let  $p_j(z):=\PP(Z_j=z)$. Then the following holds:
\begin{align*}
\mut(Z;&S=s)\leq \sum_{i,z}\sum_{i',z'}p_s(i,z)p_s(i',z')\left[\II_{\{z\equiv z'\}}\frac{p_{i\cup i'}(z\cup z')}{p_i(z)p_{i'}(z')}-1\right].
\end{align*}
\end{lemma}
Note that Lemma~\ref{lemma:Ibndgen} does not make any assumption regarding the distribution of $\bZ$ or of the user-data $(I,Z)$ -- assuming $\bZ$ is uniformly drawn from $\{0,1\}^N$, we get back Lemma \ref{lemma:Ibnd}.

\begin{proof}[Proof Outline for Theorem~\ref{thm:adaptive}]
We consider the system when $T-1$ sketches have been released, and denote by $\PP^T$ the probability distribution conditionally on the previously observed sketch values. We want to develop bounds of the form $\mut(Z;S_1^T)\le \delta_T$, for a suitable function $\delta_T$. These bounds are obtained inductively as follows. First, we can expand and bound the mutual information as follows: 
\begin{align}
\label{eq:recursion}	
\mut(Z;S_1^T)&=\sum_{t=1}^T\mut(Z;S_t|S_1^{t-1})\\
&\leq\mut(Z;S_1^{T-1})+\sup_{s,s_1^{T-1}}\mut(Z;S_T=s|S_1^{T-1}=s_1^{T-1}),
\end{align}
where $\mut(U;V=v|W=w)$ is the mutual information between $U$ and $V=v$ conditioned on $W=w$. 

Now consider any sequence $\{s_1^{T-1},s\}$. We define $\mathbb{P}^T$ to be the probability measure conditional on $S_1^{T-1}=s_1^{T-1}$, and  $p^{T+1}(i,z)=\mathbb{P}^{T+1}[(I,Z)=(i,z)]$ and $p^{T+1}_i(z)=\mathbb{P}^{T+1}[Z(i)=z]$. Using Lemma~\ref{lemma:uncondition}, we have $p^{T+1}(i,z_i)=f_i(z_i) \frac{1}{N}p^{T}_i(z_i)$ where $f_i(z_i)$ belongs to $[1-\epsilon',1+\epsilon']$ where $\epsilon'=e^{\epsilon}-1$. Further, we can use Lemma \ref{lemma:Ibndgen} to obtain:
\begin{align*}
&\mut(Z;S_T=s|S_1^{T-1}=s_1^{T-1})\leq\\
&\sum_{i_1,z_1}\sum_{i_2,z_2}p^{T+1}(i_1,z_1)p^{T+1}(i_2,z_2)\left[\II_{z_1\equiv z_2}\frac{p^{T}_{i_1\cup i_2}(z_1\cup z_2)}{p^{T}_{i_1}(z_1)p^T_{i_2}(z_2)}-1\right]
\end{align*}
\noindent Combining and rearranging the above results, we get:
\begin{equation*}
\mut(Z;S_T=s|S_1^{T-1}=s_1^{T-1})\le \frac{1}{N^2}\mbox{Var}^T\left[\sum_{i=1}^N f_i(Z_i)\right],
\end{equation*}
where $\mbox{Var}^T$ is defined w.r.t. the $\mathbb{P}^T$ measure.

Next, let $\PP^0$ be the unconditional probability, under which the $Z_i$ are i.i.d. uniform on $\{0,1\}$, and define $F:=\sum_{i=1}^N f_i(Z_i)$. Note that under $\PP^0$, the random variable $F$ has variance $\leq2\epsilon'^2 N$ -- a similar bound for $\mbox{Var}^T\left[F\right]$ would yield an upper bound of order $1/N$ on $\mut(Z;S_1^T)$. This appears difficult, as the only information we have about $\PP^T$ is that the sketches $S_1^T$ are obtained via local-DP mechanisms. However, we show that we can control $\mbox{Var}^T\left[F\right]$ via controlling the mutual information leakage. The crux of our argument is encapsulated in the following technical lemma:
\begin{lemma}
\label{lemma:Htovar}
If $\mut(Z;S_1^T)\leq\delta$, then we have: 
\begin{equation*}
\mbox{Var}^T[F]\leq\mbox{Var}^0[F]\cdot\max\left\{20,10\delta\right\}.
\end{equation*}
\end{lemma}
Lemma \ref{lemma:Htovar} is of independent interest, and could enable extensions of our result (e.g. relaxing the assumption that $w=1$). For ease of exposition, we defer the proof to Appendix \ref{app:adapt}. We instead now show how to complete the proof of Theorem \ref{thm:adaptive} via an induction argument.

Assume $\mut(Z;S_1^T)\leq\delta_T$ -- Lemma \ref{lemma:Htovar} now gives us that $\mbox{Var}^T[F]\leq\mbox{Var}^0[F]\cdot\max\left\{20,10\delta\right\}$. Now, using equation \ref{eq:recursion}, we can recursively define $\delta_{T+1}$ as: 
\begin{align*}
\delta_{T+1}= \delta_{T}+\frac{1}{N^2}\mbox{Var}^0[F]\max\left\{20,10\delta_T\right\}.
\end{align*}
Recalling that $\mbox{Var}^0[F]\le 2N\epsilon'^2$, we can bound this as:
\begin{equation*}
\delta_{T+1}\leq \delta_{T}+\frac{C}{N}\max\{1,\delta_{T}\},
\end{equation*}
where $C:=40\epsilon'^2$ is independent of $N$. It then follows that:
$\delta_T= C T/N$ for $T\le N/C$, and for $T>N/C$ one has
\begin{equation*}
\delta_T\le \left(1+\frac{C}{N}\right)^{T}.
\end{equation*}
Thus for any fixed exponent $\alpha>0$, in order to learn $N^{\alpha}$ bits of information about the unknown labels $Z_1^N$, one needs at least $T= \alpha \log(N)/\log(1+C/N)=\Omega(N\log N)$ samples.
\end{proof}	
	
We leave it as a topic for further research to establish how sharp this lower bound is. In particular, if it can be tightened to a lower bound of $\Omega(N^2)$ and further extended to $\Omega(N^2/w)$ for $w\ne 1$, this would imply that MaxSense is optimal even when one can use adaptive queries. If on the other hand there is a gap between non-adaptive and adaptive complexities, then this implies that schemes superior to MaxSense in the adaptive case have yet to be identified.

\section{Conclusion}
\label{sec:extn}

We have initiated a study in the design of recommender systems under local-DP constraints. We have provided lower bounds on the sample-complexity in both information-rich and information-scarce regime, quantifying the effect of limited information on private learning. Further, we showed tightness of these results by designing the MaxSense algorithm, which recovers the item clustering under privacy constraints with optimal sample-complexity. The lower bound techniques naturally extend to cover model selection for more general (finite) hypothesis classes, while $1$-bit sketches appear appropriate for designing efficient algorithms for the same. Development of such algorithms and analysis of matching lower bounds by leveraging and extending the techniques we introduce seem promising future research directions.





\pagebreak

\appendix

\section{Lower Bounds: Private Learning with Distortion}
\label{app:misc}

The item clustering problem fits in a more general framework of model selection from finite hypothesis-classes, with local-DP constraints: we consider a hypothesis class $\mathcal{H}, |\mathcal{H}|=M$, indexed by $[M]$. Given a hypothesis $Z$, samples $\mathbf{X}_1^U$ are drawn in an i.i.d. manner according to some distribution $P_{\mathcal{H}}(Z)$ (in our case, $u\in[U]$ corresponds to a user, and $X_u$ the ratings drawn according to the statistical model in Section~\ref{ssec:model}. $P_{\mathcal{H}}(Z)$ thus includes both the sampling of items by a user, as well as the ratings given for the sampled items). Let $\widehat{\mathbf{X}}_1^U$ be a privatized version of this data, where for each $u\in[U]$, the output $\widehat{X_u}$ is $\epsilon$-differentially private with respect to the data $X_u$ (by local-DP). Note here that $X_u$ and $\widehat{X}_u$ need not belong to the same space (for example, in the case of the Multi-MaxSense algorithm, $X_u$ is a subset of items and their ratings, while $\widehat{X}_u$ is the collection of privatized responses to the multiple MaxSense queries). Note also that the probability transition kernel $P_{\mathcal{H}}$ can be known to the algorithm (although the exact model $Z$ is unknown). Finally the learning algorithm infers the underlying model from the privatized samples. We can represent this as the Markov chain:
\begin{equation*}
Z\in\mathcal{H}\xrightarrow{\mbox{Sampling}} \mathbf{X}_1^U\xrightarrow{\mbox{Privatization}}\widehat{\mathbf{X}}_1^U\xrightarrow[\mbox{Selection}]{\mbox{Model}}\widehat{Z}
\end{equation*}
In Section \ref{sec:rich}, we considered an algorithm to be successful only if $\widehat{Z}=Z$, i.e., the model is identified perfectly. A natural relaxation of this is in terms of a distortion metric, as follows: given a distance function $d:\mathcal{Z}\times\mathcal{Z}\rightarrow\mathcal{R}_+$, we say the learner is successful if, for a given $d>0$, we have:
\begin{equation*}
d(Z,\widehat{Z})\leq d.
\end{equation*}
For any $h\in\mathcal{H}$, we define the set $B_d(h)\triangleq\{h'\in\mathcal{H}|d(h,h')\leq d\}$. Further, we define $M_d=\max_{h\in\mathcal{H}}|B_d(h)|$ to be the largest size of such a set. Finally, given a distribution for $Z$, we define the average error probability $P_e$ for a learning algorithm for the hypothesis class $\mathcal{H}$ as:
\begin{equation*}
P_e=\PP\left[d(\widehat{Z},Z)>d\right]\, .
\end{equation*}
Then we have the following bound on $P_e$:
\begin{lemma} \textbf{(Generalized Fano's Inequality)}
\label{applemma:fano}
Given a hypothesis $Z$ drawn uniformly from $\mathcal{H}$, for any learning algorithm, the average error probability satisfies:
\begin{equation*}
P_e\geq1- \frac{I(Z;\widehat{\mathbf{X}}_1^U)+1}{\log M-\log M_d}
 \, .
\end{equation*}
\end{lemma}

Lemma \ref{applemma:fano} is standard in deriving lower bounds for model-selection with distortion constraints -- for example, refer \cite{SantWain}. We present the proof for the sake of completeness:

\begin{proof}
First, we define an error indicator $E$ as:
\begin{equation*}
E=
\begin{dcases*} 
1 & : $d(Z,\widehat{Z})> d$\\
0 & : otherwise
\end{dcases*},
\end{equation*}
and hence $P_e=\mathbb{P}[E=1]$. Recall that the entropy is given by $H(x)=-x\log (x)-(1-x)\log (1-x)$. Now we have:
\begin{align*}
I(Z;\widehat{\mathbf{X}}_1^U)&\geq I(Z;\widehat{Z})\quad\mbox{(By the Data Processing Inequality)}\\
&=  H(Z)-H(Z|\widehat{Z})\\
&\geq \log M-H(Z|\widehat{Z},E)-H(E|\widehat{Z}),
\end{align*}
where the last inequality follows from  basic information inequalities, and the fact that $Z$ is uniform over $\mathcal{H}\equiv[M]$. Let us denote $\overline{P_e}=1-P_e$. Expanding the RHS, we have:
\begin{align*}
I(Z;\widehat{\mathbf{X}}_1^U)&\geq\log M-P_eH(Z|\widehat{Z},E=1)-\overline{P_e}H(Z|\widehat{Z},E=0)-1\\
&\hspace{0.5cm}\mbox{(Since $H(P_e)\geq H(E|\widehat{Z})$ and $H(P_e)\leq 1$)}\\
&\geq \overline{P_e}(\log M-H(Z|\widehat{Z},E=0))-1 
\quad\mbox{(Since $H(Z|\widehat{Z},E=1)\leq\log M$)}\\
&\geq \overline{P_e}\left(\log M-\log M_d\right)-1
\quad\mbox{(Since $H(Z|\widehat{Z},E=0)\leq\log|B_d(\widehat{Z})|\leq\log M_d$)}
\end{align*}
Rearranging, we have:
\begin{align*}
P_e&\geq1- \frac{I(Z;\widehat{\mathbf{X}}_1^U)+1}{\log M-\log M_d}
\end{align*}
\end{proof}
\noindent We now have two immediate corollaries of this lemma. First, we consider the non-adaptive learning case, i.e., where the data of each user $\widehat{X_u}$ is obtained in an i.i.d manner. Then we have:
\begin{corollary}
\label{corr:nonad}
Given a hypothesis $Z$ drawn uniformly from $\mathcal{H}$, for any non-adaptive learning algorithm, the number of users satisfies:
\begin{equation*}
P_e\geq1-\left(\frac{UI(Z;\widehat{X}_u)+1}{\log M-\log M_d}\right) \, .
\end{equation*}
\end{corollary}
\noindent Next, using Lemma \ref{lemma:basicIbnd}, we get a bound on the sample complexity of learning under local-DP. 
\begin{corollary}
Given a hypothesis $Z$ drawn uniformly from $\mathcal{H}$, for any learning algorithm on $U$ privatized samples, each obtained via $\epsilon$-local-DP, the average error probability satisfies:
\begin{equation*}
P_e\geq1-\frac{1}{\ln 2}\left(\frac{U\epsilon +1}{\log M-\log M_d}\right)\,.
\end{equation*}
\end{corollary}

Note that these results do not imply that we are assuming a prior on the hypothesis class for our algorithms; rather, the lower bound can be viewed as a probabilistic argument that shows that below a certain sample complexity, any learner fails to distinguish between a large fraction of all possible models. 

Returning to our problem of learning item clusters, we note that $M=\frac{K^N}{K!}$ in that case. Further, by choosing $d$ as the edit distance (Hamming distance) between two clusterings of items (i.e., for two clusterings $C_N$ and $C_N'$, $d(C_N,C_N')$ is the the number of items that are mapped to different clusters in the two clusterings), we get that:
\begin{align*}
M_d&=\frac{1}{K!}\sum_{i=0}^d\binom{N}{i}(K-1)^i\\
&=\frac{K^N}{K!}\PP\left[\mbox{Binomial}(N,1/K)\geq N-d\right]\\
&\leq\frac{K^N}{K!}\exp\left(\frac{-NK(1-\frac{d}{N}-\frac{1}{K})^2}{3}\right)
\end{align*}
Now, combining the above results, we obtain a more general version of Theorem \ref{Thm:basiclb}. 
\begin{theorem} 
Suppose the underlying clustering $C_N(\cdot):[M]\rightarrow[K]$ is drawn uniformly at random from $\{0,1\}^N$. Further, for a given tolerance $d>0$ and error threshold $p_{\max}$, we define a learning algorithm to be unreliable for the hypothesis class $\mathcal{H}$ if:
$$\max_{h\in[M]}\PP\left[d(\widehat{Z},Z)> d\right]>p_{\max}.$$
Then any learning algorithm that obeys $\epsilon$-local-DP is unreliable if the number of queries $U$ satisfies:
\begin{equation*}
U<(1-p_{\max})\left(\frac{NK(1-\frac{d}{N}-\frac{1}{K})^2}{3\epsilon}\right)\,.
\end{equation*}
\end{theorem}

\section{Analysis of the Pairwise Preference Algorithm}
\label{app:PP}

In this appendix, we present a complete proof for the performance of the Pairwise Preference Algorithm from Section \ref{ssec:PP}. For convenience, we first restate the theorem:

\begin{theorem} 
\label{appThm:PP}
(Theorem \ref{Thm:pairpref} in the paper) 
The Pairwise-Preference algorithm 
satisfies $\epsilon$-local-DP. Further, suppose the eigenvalues and eigenvectors of $\widehat{A}$ satisfy the following non-degeneracy conditions:
\begin{itemize}
\item The $L$ largest magnitude eigenvalues of $A$ have distinct absolute values.
\item The corresponding eigenvectors $y_1,y_2,\ldots,y_L$, normalized under the $\alpha$-norm, $||y||^2_{\alpha}=\sum_{k=1}^K\alpha_ky_k^2$, for some $\alpha$ satisfy:
\begin{equation*}
t_i\neq t_j\quad,1\leq i<j\leq L
\end{equation*}
where $t_i:= (y_1(i),\ldots ,y_L(i))$.
\end{itemize} 
Then, in the information-rich regime (i.e., when $w=\Omega(N)$), there exists $c>0$ such that the item clustering is successful with high probability if the number of users satisfies:
\begin{equation*}
U\geq c\left(N\log N\right) \, .
\end{equation*}
\end{theorem}
 
\begin{proof}
As mentioned before, privacy for the algorithm is guaranteed by the use of $\epsilon$-DP bit release (Proposition \ref{prop:1bitpriv}), and the composition property of DP (Proposition \ref{prop:composition}). 

We will prove the sample complexity bound for the case where $w=\Omega(N)$ -- the case where rated items are not private follows similarly. From the definition of the $\epsilon$-DP bit release mechanism, we have that:
$$\PP[S_u=1]=\frac{1+(e^{\epsilon}-1)\PP[S_u^0=1]}{e^{\epsilon}+1},$$
and thus for any pair of items $\{i,j\}$, defining $b_{ij}\triangleq\sum_{k=1}^K\alpha_k(b_{ki}b_{kj}+(1-b_{ki})(1-b_{kj}))$ (i.e., the probability that a random user has identical preference for items $i$ and $j$) and $\overline{b_{ij}}=1-b_{ij}$, we have:
\begin{align*}
\PP[S_u=1,P_u=\{i,j\}]&=\frac{1}{N(N-1)}\left(\frac{1}{e^{\epsilon}+1}+\left(\frac{e^{\epsilon}-1}{e^{\epsilon}+1}\right)\frac{w(w-1)}{N(N-1)}b_{ij}\right)\triangleq\frac{b_{ij}'}{N(N-1)},\\
\PP[S_u=0,P_u=\{i,j\}]&=\frac{1}{N(N-1)}\left(\frac{e^{\epsilon}}{e^{\epsilon}+1}+\left(\frac{e^{\epsilon}-1}{e^{\epsilon}+1}\right)\frac{w(w-1)}{N(N-1)}(\overline{b_{ij}}-1)\right)\triangleq\frac{\overline{b_{ij}'}}{N(N-1)},
\end{align*}
where, under the assumptions that $w=\Omega(N)$ and $\epsilon=\Theta(1)$, we have that $b_{ij}',\overline{b_{ij}}'$ are both $\Theta(1)$. Now, since $\widehat{A}_{ij}=\sum_{u\in\mathcal{U}|P_u=\{i,j\}}S_u$, we have that:
\begin{equation*}
\widehat{A}_{ij}\sim\mbox{Binomial}\left(U,\frac{b_{ij}}{N(N-1)}\right)
\end{equation*}
Setting $U=cN\log N$, we have that:
\begin{align*}
\PP[\widehat{A}_{ij}>0]&=1-\left(1-\frac{b_{ij}}{N(N-1)}\right)^U
=\frac{Ub_{ij}}{N(N-1)}+\Theta\left(\frac{U^2}{N^4}\right)
=c'b_{ij}\frac{\log N}{N}+\Theta\left(\frac{(\log N)^2}{N^4}\right)
\end{align*}

\noindent Thus we can interpret $\widehat{A}$ as representing the edges of a random graph over the item set, with an edge between an item in class $i$ and another in class $j$ if $\widehat{A}_{ij}>0$; the probability of such an edge is $\Theta\left(\frac{b_{ij}\log N}{N}\right)$. We can now use Theorem $1$ from \cite{TomoMass} to complete the proof. 
\end{proof}

\section{Lower bounds for the Information-scarce Setting}
\label{app:lbound}

In this appendix, we provide generalizations and complete proofs for the results in Section \ref{sec:LBscarce}. 

Recall that we consider a scenario where there is a single class of users, and each item is ranked either $0$ or $1$ deterministically by each user. $C_N(\cdot):[N]\rightarrow\{0,1\}$ is the underlying clustering function. We assume that the user-data for user $u$ is given by $X_u=(I_u,Z_u)$, where $I_u$ is a size $w$ subset of $[N]$ representing items rated by user $u$, and $Z_u$ are the ratings for the corresponding items; in this case, $Z_u=\{\bZ(i)\}_{i\in I_u}$. We also denote the privatized sketch from user $u$ as $S_u\in\mathcal{S}$, where ${\mathcal S}$ denotes the space from which sketches are drawn, which we assume to be finite or countably infinite. The sketch is assumed to obey $\epsilon$-DP. Finally, we assume that $\bZ$ is chosen uniformly over $\{0,1\}^N$, and the set of items $I_u$ rated by user $u$ is also assumed to be chosen uniformly at random from amongst all size-$w$ subsets of $[N]$.

\subsection{Mutual Information under Generalized Channel Mismatch}
\label{appssec:2lemmas}

Recall we define $[N]_w$ to be the collection of all size-$w$ subsets of $[N]=\{1,2,\ldots,N\}$, and  $\mathcal{D}\triangleq[N]_w\times \{0,1\}^w$ to be the set from which user information (i.e., $(I,Z)$) is drawn (and define $D=|\mathcal{D}|=\binom{N}{w}2^w$). Finally $\EE_X[\cdot]$ indicates that the expectation is over the random variable $X$. We now establish a generalization of Lemma \ref{lemma:Ibnd}.
\begin{lemma} 
\label{applemma:Ibnd}
Assume that under probability distribution $\PP$, the set $I$ of items whose type is available to a given user is independent of the type vector $Z$. Denote $p_s(i,z):=\PP((I,Z_I)=(i,z)|S=s)$. Also, for subsets $j\subset[N]$, we denote  $p_j(z):=\PP(Z_j=z)$. Then the following holds:
\begin{align*}
\mut(Z;S=s)\leq \nonumber \sum_{i,z}\sum_{i',z'}p_s(i,z)p_s(i',z')\left[\II_{z\equiv z'}\frac{p_{i\cup i'}(z\cup z')}{p_i(z)p_{i'}(z')}-1\right].
\end{align*}
\end{lemma}

Note that in the above lemma we do not make any assumption regarding: $i)$ the distribution of $\bZ$, $ii)$ the distribution of the user-data $(I,Z)$. If $\bZ$ is uniformly distributed on $\{0,1\}^N$, we recover Lemma \ref{lemma:Ibnd}.

\begin{proof}[Proof of Lemma \ref{applemma:Ibnd}]
From the definition of mutual information, we have:
\begin{align*}
\mut(\bZ;S)
&=\sum_{\bz,s}\PP[(\bZ,S)=(\bz,s)]\log\left(\frac{\PP[(\bZ,S)=(\bz,s)]}{\PP[\bZ=\bz]\PP[S=s]}\right)
=\EE_S\left[\mut(Z;S=s)\right], 
\end{align*}
where we use the notation:
$$\mut(Z;S=s):=\sum_{\bz}\PP[\bZ=\bz|S=s]\log\left(\frac{\PP[\bZ=\bz|S=s]}{\PP[\bZ=\bz]}\right)$$
Now note that:
\begin{align}
\PP[\bZ=\bz|S=s]&=\sum_{(i_1,z_1)}\PP[\bZ=\bz,(I_1,Z_1)=(i_1,z_1)|S=s]\nonumber\\
&=\sum_{(i_1,z_1)}\PP[\bZ=\bz|i_1,z_1]\PP[(I_1,Z_1)=(i_1,z_1)|s]\nonumber\\
&=\sum_{(i_1,z_1)}p_s(i,z_i)\frac{\PP[\bZ=\bz]}{p_i(z_i)}\II_{\{z\equiv z_1\}}.\nonumber
\end{align}
Combining the equations, we get
\begin{align*}
\mut(Z;S=s)&=\sum_z\sum_{i_1,z_1}\II_{z\equiv z_1}\PP(Z=z)\frac{p_s(i_1,z_1)}{p_{i_1}(z_1)}
\log\left(\sum_{i_2,z_2}\II_{z\equiv z_2}\frac{p_s(i_2,z_2)}{p_{i_2}(z_2)}\right).
\end{align*}
Using Jensen's inequality, the R.H.S. is upped bounded by the corresponding expression where averaging over $z$ conditionally on $Z_{i_1}=z_1$  is taken inside the logarithm, yielding
\begin{align*}
\mut(Z;S=s)
&\le \sum_{i_1,z_1}p_s(i_1,z_1)\log\left(\sum_z \II_{z\equiv z_1}\frac{\PP(Z=z)}{p_{i_1}(z_1)}
\sum_{i_2,z_2}p_s(i_2,z_2)\frac{\II_{z\equiv z_2}}{p_{i_2}(z_2)}\right)\\
&=\sum_{i_1,z_1}p_s(i_1,z_1)\log\left(\sum_{i_2,z_2}p_s(i_2,z_2)\II_{z_1\equiv z_2}\frac{p_{i_1\cup i_2}(z_1\cup z_2)}{p_{i_1}(z_1)p_{i_2}(z_2)}\right).
\end{align*}
The result now follows from the inequality $\log(x)\le x-1$.
\end{proof}

\subsection{Lower Bound on Scaling for Clustering with Local-DP}
\label{appssec:LB}

We now fill in the proofs for results from Section \ref{ssec:LB}:
\begin{lemma}
(Lemma \ref{lemma:combbnd} in the paper)
If $w=o(N)$, then:
\begin{equation*}
\left|\frac{\binom{N-w}{w}}{\binom{N}{w}}-\left(1-\frac{w^2}{N}\right)\right|=\Theta\left(\frac{w^4}{N^2}\right)
\end{equation*}
\end{lemma}
\begin{proof}
First, it is easy to verify that the binomial coefficients satisfy:
\begin{align*}
\Rightarrow\left(1-\frac{w}{N-w+1}\right)^w&\leq\frac{\binom{N-w}{w}}{\binom{N}{w}}\leq\left(1-\frac{w}{N}\right)^w
\end{align*}
Now for the upper bound, using the binomial expansion, we have:
\begin{align*}
\left(1-\frac{w}{N}\right)^w
&=1-\frac{w^2}{N}+\Theta\left(\frac{w^4}{N^2}\right)
\end{align*}
Similarly for the lower bound, we have:
\begin{align*}
\Bigg(1-\frac{w}{N-w+1}\Bigg)^w
&=1-\frac{w^2}{N-w+1}+\frac{w^4}{2(N-w+1)^2}-\ldots\\
&\geq 1-\frac{w^2}{N}-\frac{w^3}{N(N-w+1)}+\frac{w^4}{2(N-w+1)^2}-\ldots\\
&=1-\frac{w^2}{N}-\Theta\left(\frac{w^4}{N^2}\right)
\end{align*}
\end{proof}

\begin{theorem} 
(Theorem \ref{Thm:stronglb} in the paper)
{\it Under the scaling assumption $w=o(N^{1/3})$, and for $\epsilon<\ln(2)$, it holds that
\begin{equation*}
\mut(\bZ,S)=O\left(\frac{w}{N}\right)\cdot
\end{equation*}
and thus there exists a constant $c>0$ such that any cluster learning algorithm with local-DP is unreliable if the number of users satisfies:
\begin{equation*}
U < c\left(\frac{N^2}{w}\right).
\end{equation*}
}
\end{theorem}

\begin{proof}
In the proof of Theorem \ref{Thm:weaklb}, the two steps which are weak are the conversion to the base measure $\PP^0[]$ using Lemma \ref{lemma:uncondition}, and the evaluation of the bound for $\Delta_1$. We start off by performing a similar decomposition of the bound, but without first converting to the base measure. For any $S=s$, we have:
\begin{align*}
\EE\left[2^{|I_1\cap I_2|}\II_{\{Z_1\equiv Z_2}-1\right]
&=\sum_{\ell=1}^N\EE\left[\II_{\{I_1\cap I_2=\{\ell\}\}}(2*\II_{\{Z_1\equiv Z_2\}}-1)\right]\\
&+\EE\left[\II_{\{|I_1\cap I_2|>1\}}(2^{|I_1\cap I_2|}\II_{\{Z_1\equiv Z_2\}}-1)\right]\\
&=\Delta_1'+\Delta_1''+\Delta_2'
\end{align*}
where 
\begin{align*}
\Delta_1'&=\sum_{\ell=1}^N\EE\left[\II_{\{\ell \in I_1\cap I_2\}}(2*\II_{Z_1(\ell)=Z_2(\ell)}-1)\right]\\
\Delta_1''&=-\sum_{\ell=1}^N\EE\left[\II_{\{\ell \in I_1\cap I_2; |I_1\cap I_2|>1\}}(2*\II_{Z_1(\ell)=Z_2(\ell)}-1)\right],\\ 
\Delta_2'&=\EE\left[\II_{|I_1\cap I_2|>1}(2^{|I_1\cap I_2|}\II_{\{Z_1\equiv Z_2\}}-1)\right]
\end{align*}
Note that $\Delta_1'+\Delta_1''$ are similar to $\Delta_1$ and $\Delta_2'$ similar to $\Delta_2$ in Theorem \ref{Thm:weaklb} (albeit without first converting to the base measure). Unlike before, however, we first bound $\Delta_1''+\Delta_2'$, establishing that $\Delta_1''+\Delta_2'=O(w^4/N^2)=o(w/N)$ whenever $w=o(N^{1/3})$. For $\Delta_1'$, we need to employ a more sophisticated technique for bounding.
As before, we write $\PP^0$ for the base probability distribution under which $(I_1,Z_1)$ and $(I_2,Z_2)$  are independent and uniformly distributed over $\mathcal{D}$, and denote by $\EE^0$ mathematical expectation under $\PP^0$. 
For $\Delta_1''$, we have:
\begin{align*}
\Delta_1'' &\leq\sum_{\ell=1}^N\EE\left[\II_{\{\ell \in I_1\cap I_2; |I_1\cap I_2|>1\}}\right]\\
&=\EE\left[|I_1\cap I_2|\II_{\{|I_1\cap I_2|>1\}}\right]
\end{align*}
Since the RHS is non-negative, we use Lemma \ref{lemma:uncondition} to convert the expectation to the base measure. Thus, we get:
\begin{align}
\Delta_1'' &\leq e^{2\epsilon}\left[\EE^0\left[|I_1\cap I_2|\right]-\PP^0\left[|I_1\cap I_2|=1\right]\right]\nonumber\\
&=e^{2\epsilon}\frac{w\binom{N-1}{w-1}-w\binom{N-w}{w-1}}{\binom{N}{w}}\nonumber\\
&=e^{2\epsilon}\left(\frac{w^2}{N}-\left(\frac{w^2}{N-2w+1}\right)\frac{\binom{N-w}{w}}{\binom{N}{w}}\right)\label{eq:delta1bnd2}
\end{align}
Similarly for $\Delta_2'$, we have:
\begin{align*}
\Delta_2' &\leq\EE\left[\II_{\{|I_1\cap I_2|>1\}}2^{|I_1\cap I_2|}\II_{\{Z_1\equiv Z_2\}}\right]
\leq e^{2\epsilon}\EE^0\left[\II_{\{|I_1\cap I_2|>1\}}2^{|I_1\cap I_2|}\II_{\{Z_1\equiv Z_2\}}\right]\\
&\le e^{2\epsilon}\PP^0\left[|I_1\cap I_2|>1\right],
\end{align*}
as $\PP^0\left[Z_1\equiv Z_2\right]=2^{-|I_1\cap I_2|}$. Now since $I_1$ and $I_2$ are picked independently and uniformly over all size $w$ subsets of $[N]$ (under $\PP^0$), we have:
\begin{align}
\Delta_2' &\leq e^{2\epsilon}\left(1-\frac{\binom{N-w}{w}+w\binom{N-w}{w-1}}{\binom{N}{w}}\right)
=e^{2\epsilon}\left(1-\left(1+\frac{w^2}{N-2w+1}\right)\frac{\binom{N-w}{w}}{\binom{N}{w}}\right)\label{eq:delta2bnd2}
\end{align}
Finally combining equations (\ref{eq:delta1bnd2}) and (\ref{eq:delta2bnd2}), we get:
\begin{equation*}
\Delta_1''+\Delta_2'\leq e^{2\epsilon}\left(1+\frac{w^2}{N}-\left(1+\frac{2w^2}{N-2w+1}\right)\frac{\binom{N-w}{w}}{\binom{N}{w}}\right),
\end{equation*}
and using Lemma \ref{lemma:combbnd}, we get:
\begin{align*}
\Delta_1''+\Delta_2'&\leq e^{2\epsilon}\left(1+\frac{w^2}{N}-\left(1+\frac{2w^2}{N}+\frac{2w^2(2w-1)}{N(N-2w+1)}\right)\left(1-\frac{w^2}{N}-O\left(\frac{w^3}{N^2}\right)\right)\right)\\
&=O\left(\frac{w^4}{N^2}\right)
\end{align*}

Thus, we now have:
\begin{equation*}
\EE\left[2^{|I\cap J|}\II_{\{Z(\ell)
=Z'(\ell)\forall\ell\in I\cap J\}}-1\right]\leq\sum_{\ell=1}^N\EE\left[\II_{\{\ell \in I\cap J\}}(2*\II_{Z(\ell)=Z'(\ell)}-1)\right]+O\left(\frac{w^4}{N^2}\right)
\end{equation*}

Under the scaling assumption $w=o(N^{1/3})$, the second term in the right-hand side of the above equation is $o(w/N)$, and we only need to establish that the first term in the right-hand side is $O(w/N)$. 

As in Theorem \ref{Thm:one-bit}, we introduce the notation $\PP(\ell \in I, Z(\ell)=\sigma|S=s)=\pi_{\ell,\sigma}$ (here we can omit indexing with respect to $s$ for notational convenience). The following identity is then easily established:
\begin{equation}
\label{eq:2}
\sum_{\ell=1}^N \EE \II_{\{\ell\in I_1\cap I_2\}}\left(2\II_{Z_1(\ell)=Z_2(\ell)\}}-1\right)= \sum_{\ell=1}^N\left(\pi_{\ell,0}-\pi_{\ell,1}\right)^2.
\end{equation}
The left-hand side of (\ref{eq:2}) is thus a non-negative definite quadratic form of the variables 
$$
p_{i,z}:=\PP(I=i,Z=z|S=s),
$$
where we have that $\pi_{\ell,\sigma}=\sum_{i,\sigma|\ell\in i,z(\ell)=\sigma}p_{i,z}$ in (\ref{eq:2}).
We know however by Lemma \ref{lemma:uncondition} that these variables are  constrained to lie in the convex set defined by the following inequalities:
\begin{align*}
\sum_{(i,z)\in\mathcal{D}}p_{i,z}&=1,\quad
\frac{e^{-\epsilon}}{D}\le p_{i,z}\le \frac{e^{\epsilon}}{D}.
\end{align*}
Defining $\epsilon':=e^{\epsilon}-1=\max(e^{\epsilon}-1,1-e^{-\epsilon})$, we can relax the last constraint to 
$$
1-\epsilon'\le p_{i,z}D\le 1+\epsilon'.
$$
Provided $\epsilon$ is small enough (precisely, provided $\epsilon<\ln(2)$, which we have assumed), it holds that $\epsilon'<1$. 

Given this setup, we can now formulate the problem of upper bounding $\Delta_1'$ as the following optimization problem:
\begin{equation}
\label{eq:optbound}
\begin{aligned}
& \underset{\{p_{i,z}\}_{(i,z)\in\mathcal{D}}}{\text{maximize}}
&& \sum_{\ell=1}^N\left(\pi_{\ell,0}-\pi_{\ell,1}\right)^2 \\
& \text{subject to}
&& \sum_{(i,z)\in\mathcal{D}}p_{i,z} = 1, \; \\
&&& p_{i,z}D\in\left[1-\epsilon',1+\epsilon'\right].
\end{aligned}
\end{equation}
In order to evaluate this bound, we need to first characterize the extremal points of the above convex set. We do this in the following lemma.
\begin{lemma}
\label{lemma:etremalpts} 
The extremal points of the convex set of distributions $\{p_{i,z}\}$ defined by (\ref{eq:optbound}) consists precisely of the distributions $p^A_{i,z}$ indexed by the sets $A\subset \mathcal{D}$ of cardinality 
$$
|A|=\binom{N}{w}2^{w-1}=\frac{D}{2},
$$
defined by
\begin{equation}
\label{eq:4} 
p^A_{i,z}=\left\{\begin{array}{ll}
\frac{1+\epsilon'}{D}&\hbox{if $(i,z)\in A$},\\
\frac{1-\epsilon'}{D}&\hbox{if $(i,z)\notin A$}.
\end{array}
\right.
\end{equation}
\end{lemma}
\begin{proof}
Let $\{p_{i,z}\}$ be a probability distribution satisfying constraints (\ref{eq:optbound}). The aim is to establish the existence of non-negative weights $\gamma_S$ for each subset $S\subset\mathcal{D}$ of size $D/2$, summing to 1, and such that for all $(i,z)\in\mathcal{D}$, one has:
\begin{equation}
\label{eq:8}
p_{i,z}=\sum_{S\subset\mathcal{D},|S|=D/2}\gamma_S(1+\epsilon'\II_{(i,z)\in S}-\epsilon'\II_{(i,z)\notin S})/D.
\end{equation}
Let us now express the existence of such weights $\gamma_S$ as a property of a network flow problem. For each $n\in [D]$, define:
$$
\alpha_n:=\left(p_n-\frac{1-\epsilon'}{D}\right)\frac{D}{2\epsilon'}\cdot
$$
The constraint $p_n\in[(1-\epsilon')/D,(1+\epsilon')/D]$ entails that $\alpha_n\in[0,1]$. Construct now a network with for each $n\in[D]$ two links, labelled $(n\in )$ and $( n \notin)$, and with respective capacities $\alpha_n$ and $1-\alpha_n$.
In addition, for each set $S\subset [D]$, $|S|=D/2$, create a route $r_S$ through this network, which for each $n\in D$ crosses link $(n\in)$ if $n\in S$, and crosses link $(n\notin)$ if $n\notin S$. All such routes are connected to a source and a sink node. 

We now claim that the existence of probability weights $\gamma_S$ satisfying (\ref{eq:8}) is equivalent to the fact that the maximum flow through this network is equal to 1. Indeed, the existence of a flow of total weight 1 is equivalent to the existence of a probability distribution $\gamma_S$ on the routes $r_S$ through this network which match the link capacity constraints, that is to say such that for all $n\in [D]$, one has:
$$
\begin{array}{l}
\sum_{S:n\in S}\gamma_S=\alpha_n,\\
\sum_{S:n\notin S}\gamma_S=1-\alpha_n.
\end{array}
$$
It is readily seen that this condition implies (\ref{eq:8}). Conversely, if the probability weights $\gamma_S$ satisfy (\ref{eq:8}), using the definition of $\alpha_n$, it is easily seen that the two previous equations hold.

Let us now establish the existence of such a flow. To this end, we use the max flow-min cut theorem. Any set of links that contains, for some $n\in[D]$, both links $(n\in)$ and $(n\notin)$, is a cut, and its capacity is at least $\alpha_n+1-\alpha_n$, hence larger than 1. Any cut $C$ which for each $n$ either does not contain $(n\in)$ or does not contain $(n\notin)$ must be such that either:
\begin{equation}
\label{eq:9}
|C\cap\{\cup_{n\in [D]}(n\in)\}|>D/2
\end{equation} 
or: 
\begin{equation}\label{eq:10}
|C\cap\{\cup_{n\in [D]}(n\notin)\}|>D/2,
\end{equation}
for otherwise we can identify $S\subset [D]$, $|S|=D/2$ which crosses this cut $C$. Assume thus that (\ref{eq:9}) holds. Assume without loss of generality that $C$ contains the links $(n\in)$ for all $n=1,\ldots, D/2+1$. The weight of this cut is thus at least $\sum_{n=1}^{D/2+1}\alpha_n$. We now argue that this must be at least 1. Indeed, it holds that:
$$
\sum_{n=1}^D\alpha_n=D/2.
$$
However, if $\sum_{n=1}^{D/2+1}\alpha_n<1$, using the fact that each $\alpha_n$ is at most 1, it follows that $\sum_{n=1}^{D}\alpha_n$ is strictly less than $1+D/2-1=D/2$, a contradiction. The case when cut $C$ verifies Equation~(\ref{eq:10}) is similar.
\end{proof}

We can now complete the proof of Theorem~\ref{Thm:stronglb}. Since as argued the second term in the right-hand side of (\ref{eq:2}) is a non-negative definite quadratic form of the $p_{i,z}$, it is in particular a convex function of the $p_{i,z}$, and as such is maximized over the convex set described by (\ref{eq:optbound}) at one of its extremal points, which are precisely identified by Lemma~\ref{lemma:etremalpts}. It will thus suffice to establish the following inequality for all $A\subset\mathcal{D}$ of size half the cardinality of the full set:
\begin{equation}
\label{eq:5} 
\sum_{\ell=1}^N\left(\pi^A_{\ell,0}-\pi^A_{\ell,1}\right)^2\le O(w/N),
\end{equation}
where we introduced the notation for all $\ell\in[N]$ and $\sigma\in\{0,1\}$:
\begin{equation*}
\pi^A_{\ell,\sigma}=\sum_{i: \ell \in i}\sum_{z:z(\ell)=\sigma}p^A_{i,z},
\end{equation*}
and $p^A_{i,z}$ is as defined in~(\ref{eq:4}). Introducing also the sets 
$$
A_{\ell,\sigma}=\{(i,z):\ell\in i\hbox{ and } z(\ell)=\sigma\},
$$
we have
\begin{equation}
\label{eq:6}
\begin{array}{lll}
\pi^A_{\ell,0}-\pi^A_{\ell,1}&=&\frac{2\epsilon'}{\binom{N}{w}2^w}\left[|A_{\ell,0}\cap A|-|A_{\ell,1}\cap A|\right]\\
&=&\frac{2\epsilon'}{\binom{N}{w}2^w}\left<\II_A,v_{\ell}\right>
\end{array}
\end{equation}
where in the last display we used the following notations. $\left<\cdot , \cdot\right>$ stands for the scalar product in $\mathbb{R}^{D}$, $\II_A$ is the characteristic vector of the set $A$, and $v_{\ell}$ is defined as
$$
v_{\ell}(i,z)=\II_{\{\ell\in i\}}\left(1-2z(\ell)\right).
$$
Equation~(\ref{eq:6}) entails that the left-hand side of Equation~(\ref{eq:5}) also equals
\begin{equation}
\label{eq:7} 
\sum_{\ell=1}^N \left(\frac{2\epsilon'}{D}\right)^2\left<\II_A,v_{\ell}\right>^2.
\end{equation}
The scalar product $\left< v_{\ell},v_{\ell'}\right>$ reads, for $\ell\ne \ell'$:
$$
\begin{array}{lll}
\left< v_{\ell},v_{\ell'}\right>&=&\sum_{i:\ell,\ell'\in i}\sum_z (1-2 z(\ell))(1-2z(\ell'))\\
&=&\sum_{i:\ell,\ell'\in i} 2^{w-2}2\left[(1)*(1)+(1)*(-1)\right]\\
&=&0.
\end{array}
$$
Note further that for all $\ell\in[N]$, one has
$$
||v_{\ell}||^2=\binom{N-1}{w-1}2^w=\frac{wD}{N}.
$$
Orthogonality and equality of norms among the $v_{\ell}$ readily implies that the expression in (\ref{eq:7}) is upper-bounded by
$$
\left(\frac{2\epsilon'}{D}\right)^2 \frac{wD}{N}||\II_A||^2.
$$
Recalling that the vector $\II_A$ has $\frac{D}{2}$ entries equal to 1, and all other entries equal to zero, the square of its Euclidean norm $||\II_A||^2$ equals precisely $\frac{D}{2}$. Plugging this value in the last display, after cancellation, one obtains that the expression in (\ref{eq:7}) is bounded by
$$
2 \epsilon'^2\frac{w}{N}\cdot
$$
This completes the proof.
\end{proof}

\section{Lower Bound for Adaptive Queries}
\label{app:adapt}

In Section \ref{ssec:adaptive}, to establish a lower bound on the sample complexity for privacy-preserving cluster-learning with adaptive queries, we considered the following setup: we defined $\bZ\in\{0,1\}^N$ to be a random type-vector, and defined $\PP^0$ to be the unconditional probability under which the $Z_i$ are i.i.d. uniform on $\{0,1\}$. Finally, in the proof of Theorem \ref{thm:adaptive}, we were interested in a given random variable $F:=\sum_{i=1}^N f_i(Z_i)$, where $f_i(Z_i)\in[e^{-\epsilon},e^{\epsilon}]$. Note that under $\PP^0$, the random variable $F$ has variance $\leq2\epsilon'^2 N$. The crux of the proof of Theorem \ref{thm:adaptive} was based on the following technical lemma:
\begin{lemma}
\label{applemma:Htovar}
(Lemma \ref{lemma:Htovar} in the paper)
If $\mut(Z;S_1^T)\leq\delta$, then we have: 
\begin{equation*}
\mbox{Var}^T[F]\leq\mbox{Var}^0[F]\cdot\max\left\{20,10\delta\right\}.
\end{equation*}
\end{lemma}

In this appendix, we provide a proof for this result. The argument proceeds in several steps.

\noindent\textbf{Step 1: Bounding the divergence between the measure on $F$ under $\PP^T$ and under $\PP^0$}:
\begin{lemma}
For each $f$ in the support of any discrete random variable $F$, let $p_f$ and $p^0_f$ denote the probabilities that $F=f$ under $\PP^T$ and $\PP^0$ respectively. Then we have:
\begin{equation}
\label{eq.ex1}
H(\PP^0)-H(\PP^T)\geq D(p||p^0)=\sum_f p_f \log\left(\frac{p_f}{p^0_f}\right)\cdot
\end{equation}
\end{lemma}
\begin{proof}
For each $f$, let $N_f$ denote the number of vectors $z\in\{0,1\}^N$ for which $F=f$, so that $p^0_f=N_f 2^{-N}$. Now we have:
\begin{align*}
H(\PP^T)&=\sum_f p_f \sum_{z: F(z)=f}\frac{\PP^T(z)}{p_f}\left[\log\left(\frac{1}{p_f}\right)+\log\left(\frac{p_f}{\PP^T(z)}\right)\right]\\
&\le \sum_f p_f\left[\log\left(\frac{1}{p_f}\right)+\log(N_f)  \right]\\
&=\sum_f p_f \left[\log\left(\frac{1}{p_f}\right)+\log(N)+\log(p^0_f)\right]\\
&=H(\PP^0)-D(p||p^0),
\end{align*}
where the inequality follows by upper-bounding the entropy of a probability distribution on a set of size $N_f$ by $\log(N_f)$.
\end{proof}

\noindent{\bf Step 2: Bounding variance of $F$ under $\PP^T$ given divergence constraints:}

Let $\bar{F}=\EE^0[F]$ (i.e., the expectation of $F$ under $\PP^0$).  Note that:
\begin{equation*}
\hbox{Var}^T(F)=\inf_{x\in \RR}\EE^T(F-x)^2\le \EE_{\PP^T}(F-\bar{F})^2=\sum_f p_f (f-\bar{F})^2.
\end{equation*}
Assume that the entropy $H(\PP^T)$ verifies $H(\PP^T)\ge H(\PP^0)-\delta$, for some $\delta\ge 0$. Then in view of (\ref{eq.ex1}) and the previous display, an upper bound on the variance of $F$ under $\PP^T$ is provided by the solution of the following optimization problem:
\begin{eqnarray}
\hbox{Maximize}&\sum_f p_f (f-\bar{F})^2\nonumber\\
\hbox{over}&p_f \ge 0\nonumber\\
\hbox{such that}&\sum_f p_f =1\nonumber\\
\hbox {and}&\sum_f p_f \log\left(\frac{p_f}{p^0_f}\right)\le \delta. \label{eq:ex2}
\end{eqnarray}
It is readily seen (for example, by introducing the Lagrangian of this optimization problem, and a dual variable $\nu^{-1}>0$ for the constraint (\ref{eq:ex2}))that the optimal of this convex optimization problem is achieved by:
\begin{equation*}
p_f:=\frac{1}{Z(\nu)}p^0_f e^{\nu(f-\bar{F})^2},
\end{equation*}
for a suitable positive constant $\nu$, where the normalization constant $Z(\nu)$ is given by:
\begin{equation*}
Z(\nu):=\sum_f p^0_f e^{\nu(f-\bar{F})^2}=\EE^0 e^{\nu(F-\bar{F})^2}.
\end{equation*}
For this particular distribution, the divergence $D(p||p^0)$ reads:
\begin{equation*}
\sum_f \frac{1}{Z(\nu)}p^0_f e^{\nu(f-\bar{F})^2}\left[\nu(f-\bar{F})^2-\log Z(\nu)\right] = 
-\log(Z(\nu))+\frac{\nu}{Z(\nu)}\EE^0(F-\bar{F})^2e^{\nu(F-\bar{F})^2},
\end{equation*}
so that constraint (\ref{eq:ex2}) reads:
\begin{equation}
\label{eq.ex3}
-\log(Z(\nu))+\frac{\nu}{Z(\nu)}\EE^0(F-\bar{F})^2e^{\nu(F-\bar{F})^2}\le \delta.
\end{equation}
This characterization in turn allows to establish the following:
\begin{lemma}
\label{cx-analysis}
Let $\psi(\nu):=\log Z(\nu)$. Assume there exist $a$, $\nu>0$ such that: 
\begin{equation}
\label{fenchel}
\nu a -\psi(\nu)\ge \delta.
\end{equation}
Then the solution to the value of the optimization problem (\ref{eq:ex2}) is less than or equal to $a$.
\end{lemma}
\begin{proof}
Note that by H\"older's inequality, function $\psi$ is convex, so that its derivative:
$$
\psi'(\nu)=Z^{-1}(\nu)\EE_0 (F-\bar{F})^2 e^{\nu(F-\bar{F})^2},
$$
is non-decreasing. Note further that the function $\nu \psi'(\nu) -\psi(\nu)$ appearing in the left-hand side of (\ref{eq.ex3}) is non-decreasing for non-negative $\nu$, as its derivative reads $\nu \psi''(\nu)$. Thus the value $\nu^*$ which achieves the optimum is such that 
$$
\nu^*\psi'(\nu^*)-\psi(\nu^*)=\delta
$$
and the sought bound is $\psi'(\nu^*)$. Now for a given $a\in \RR$, the supremum of $\nu a -\psi(\nu)$ is achieved precisely at $\nu$ such that $a=\psi'(\nu)$. Thus if for some $\nu$ and some $a$, condition~(\ref{fenchel}) holds, it follows that:
$$
\sup_{\nu}\left(\nu a -\psi(\nu)\right)\ge \delta=\sup_{\nu}\left(\nu a^* -\psi(\nu)\right),
$$
where $a^*:=\psi'(\nu^*)$. It follows from monotonicity of $\nu\to \nu \psi'(\nu)-\psi(\nu)$ that the value $\nu'$ where the supremum is achieved in the left-hand side, and such that $a=\psi'(\nu')$, verifies $\nu'\ge \nu^*$. Monotonicity of $\psi'$ then implies that $a\ge a^*$ as announced.
\end{proof}

\noindent{\bf Step 3: Deriving explicit bounds, using concentration results under $\PP^0$}.

Consider the centered and scaled random variable:
\begin{equation*}
G:=\frac{F-\bar{F}}{\sigma}\cdot
\end{equation*}
Recall that after centering, each variable $f_i(Z_i)$ is bounded in absolute value by $\epsilon'$. Thus, using the Azuma-Hoeffding inequality yields the following bound:
\begin{equation}
\label{hoeffding}
\PP^0(G>A)\le e^{-A^2/2},\; A>0,
\end{equation}
and the same bound holds for $\PP^0(G<-A)$. To obtain the above, we used the fact that after centering, $f_i(Z_i)$ is of the form $\sigma_i (2Z_i-1)$ where $\sigma_i$ is the standard deviation of $f_i(Z_i)$ under $\PP^0$. We now apply these to bound the value of $Z(\nu)$ as follows:
\begin{lemma}
\label{concentrate}
Define $\sigma^2=\mbox{Var}^0[F]$ (i.e., under $\PP^0$), and consider any $\nu\in\left(0,\frac{1}{2\sigma^2}\right)$. Then the partition function $Z(\nu)$ verifies:
\begin{equation}\label{beah}
Z(\nu)\le 1+ \frac{4\nu\sigma^2}{1-2\nu\sigma^2}\cdot
\end{equation}
\end{lemma}
\begin{proof}
We can write:
\begin{align*}
Z(\nu)&=\int_0^{\infty}\PP^0\left(e^{\nu(F-\bar{F})^2}\ge t\right) dt
\leq 1+\int_1^{\infty}\PP^0\left(\nu(F-\bar{F})^2\ge \log t\right)dt\\
&=1+\int_0^{\infty}\PP^0\left(|G|\ge\sqrt{\frac{x}{\nu\sigma^*}}\right)e^x dx \quad\mbox{(Substituting $e^x=t$)}\\
&=1+\int_0^{\infty}\PP^0\left(|G|\ge y\right)2by e^{by^2}dy
\quad\mbox{(Denoting $b=\nu\sigma^2\in(0,1/2)$, and substituting $by^2=x$)}\\
&=1+\int_0^{\infty}\left[\PP^0(G\ge y)+\PP^0(G\le -y)\right]2by e^{by^2}dy.
\end{align*}
Using Hoeffding's bound~(\ref{hoeffding}), the last term is upper-bounded by
\begin{align*}
1+2\int_0^{\infty}e^{-y^2/2}  2by e^{by^2}dy&=1+2\left[\frac{-2b}{1-2b}e^{-(y^2/2)*(1-2b)}\right]^{\infty}_0
=1+\frac{4b}{1-2b},
\end{align*}
as announced in (\ref{beah}).
\end{proof}

Finally, using these three results, we can prove Lemma \ref{applemma:Htovar}:
\begin{proof}[Proof of Lemma \ref{applemma:Htovar}]
Fix $\delta>0$, and recall $\sigma^2:=\mbox{Var}^0[F]$. We now want find some $b>0$ such that $\mbox{Var}^T[F]\leq b\sigma^2$. In view of Lemma~\ref{cx-analysis}, it suffices to verify that for some $\nu>0$, Condition $\nu b\sigma^2 -\psi(\nu)\ge \delta$ holds. In view of Lemma~\ref{concentrate}, denoting the corresponding upper bound to $\psi(\nu)$ by: 
$$
\phi(\nu):=
\begin{cases}
\log\left(1+\frac{4\nu \sigma^2}{1-2\nu \sigma^2}\right)&\mbox{:If }\nu\sigma^2<1/2\\
+\infty&\mbox{:Otherwise}
\end{cases},
$$
it suffices to find $b$ such that for some $\nu$, $\nu b\sigma^2 - \phi(\nu)\ge \delta$. Maximizing $\nu b\sigma^2 -\phi(\nu)$ over $\nu$ for fixed $b$, one finds that the optimal value for $\nu$ is given by: 
$$
\nu^*=\frac{1}{2\sigma^2}\sqrt{1-\frac{4}{b}},
$$
Plugging this expression for $\nu^*$ in $\nu b\sigma^2 -\phi(\nu)$, we have that $b\sigma^2$ upper-bounds $\mbox{Var}^T[F]$ if:
$$
\frac{b}{2}\sqrt{1-\frac{4}{b}}-\log\left(\frac{1+(1-4/b)^{1/2}}{1-(1-4/b)^{1/2}}\right)\ge \delta.
$$
For $b\ge 16/3$, it holds that $1/2\le (1-4/b)^{1/2}\le 1$. Thus under this condition on $b$, the left-hand side of the above is at least as large as:
$$
\frac{b}{4}-\log\left(\frac{[1+(1-4/b)^{1/2}]^2}{1-1+4/b}\right)\ge \frac{b}{4}-\log(b)\ge \frac{b}{10},\quad\mbox{if $b>20$}.
$$
Thus, setting $b=\max\left\{20, 10\delta\right\}$, we see that the above conditions are satisfied.
\end{proof}

\end{document}